\documentclass[pdflatex,sn-nature]{sn-jnl}

\usepackage[english]{babel}
\usepackage{anyfontsize}
\usepackage{graphicx}%
\usepackage{multirow}%
\usepackage{threeparttable}
\usepackage{amsmath,amssymb,amsfonts}%
\usepackage{amsthm}%
\usepackage{latexsym,stmaryrd}
\usepackage{mathrsfs}%
\usepackage[title]{appendix}%
\usepackage{xcolor}%
\usepackage{textcomp}%
\usepackage{manyfoot}%
\usepackage{booktabs}%
\usepackage[algo2e,english,boxruled,noline,noend,linesnumbered]{algorithm2e}
\usepackage{mathtools}
\usepackage{xspace}
\usepackage[all]{xy}
\usepackage{pgfplots}
\pgfplotsset{compat=1.18}
\usepgfplotslibrary{statistics}
\usepackage{cleveref}
\usepackage{float}
\renewcommand{\orcidlogo}{%
  \includegraphics[width=10pt]{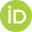}%
}
\renewcommand{\orcid}[1]{\href{https://orcid.org/#1}{\,\orcidlogo\,}}

\newcommand{\ignore}[1]{}

\ignore{
\makeatletter
\newif\if@restonecol
\makeatother

}

\newcommand{\glob}[1]{\stackrel{#1}{\to}}
\newcommand{\impl}[2]{\mathrel{\substack{#2\\\longrightarrow\\(#1)}}}

\newcommand{\as}[1]{\colorbox{purple}{\color{white}\bfseries Andreas:}{\color{purple} #1}\xspace}
\newcommand{\lk}[1]{\colorbox{teal}{\color{white}\bfseries Lucas:}{\color{teal} #1}\xspace}

\newcommand{\cO}{{\cal O}} 
\newcommand{\VV}{{\cal V}} 

\newcommand{\range}[2]{\left[#1\;..\;#2\right]}

\newcommand{\lit}[1]{\left\llbracket #1 \right\rrbracket}

\newcommand{\solv}{\operatorname{solv}}

\newcommand{\dom}{\operatorname{dom}}

\newcommand{\vars}{\operatorname{vars}}

\newcommand{\half}{\Rightarrow}

\newcommand{\false}{\mathit{false}}

\theoremstyle{thmstyleone}
\newtheorem{corollary}{Corollary}
\newtheorem{theorem}{Theorem}

\theoremstyle{thmstyletwo}
\newtheorem{example}{Example}

\theoremstyle{thmstylethree}
\newtheorem{definition}{Definition}

\begin{document}

\title{Global Difference Constraint Propagation
for Constraint Programming$^\dagger$}

\renewcommand{\thefootnote}{\fnsymbol{footnote}}
\footnotetext[2]{An earlier, much shorter version of this paper was originally published in PPDP2008~\cite{feydy2008global}}
\renewcommand{\thefootnote}{\arabic{footnote}}
\setcounter{footnote}{0}
    
\author*[1]{\fnm{Lucas} \sur{Kletzander}\orcid{0000-0002-2100-7733}}
\email{lucas.kletzander@tuwien.ac.at}

\author*[2,3]{\fnm{Jip J.} \sur{Dekker}\orcid{0000-0002-0053-6724}}
\email{jip.dekker@monash.edu}

\author*[4]{\fnm{Andreas} \sur{Schutt}\orcid{0000-0001-5452-4086}}
\email{andreas.schutt@csiro.au}

\author*[2,3]{\fnm{Peter J.} \sur{Stuckey}\orcid{0000-0003-2186-0459}}
\email{peter.stuckey@monash.edu}

\affil[1]{\orgdiv{Databases and Artificial Intelligence Group}, \orgname{TU Wien}, \orgaddress{\street{Karlsplatz 13}, \city{Wien}, \postcode{1040}, \country{Austria}}}

\affil[2]{\orgdiv{Department of Data Science and Artificial Intelligence}, \orgname{Monash University}, \orgaddress{\city{Clayton}, \postcode{3168}, \state{Victoria}, \country{Australia}}}

\affil[3]{\orgname{ARC Training Centre in Optimisation Technologies, Integrated Methodologies, and Applications (OPTIMA)}, \orgaddress{\city{Melbourne}, \state{Victoria}, \country{Australia}}}

\affil[4]{\orgname{CSIRO Technology}, \orgaddress{\city{Clayton}, \postcode{3168}, \state{Victoria}, \country{Australia}}}

\abstract{
Difference constraints of the form $x - y \leq d$ are well studied,
with efficient algorithms for satisfaction and implication, because of their
connection to shortest paths.
Finite domain propagation algorithms, however, typically do not make use of these
algorithms, and treat each difference constraint as a separate
propagator.
Propagation does guarantee completeness of solving, but can be needlessly slow.
In this paper we describe how to build a (bounds consistent) global propagator for difference constraints that treats them all simultaneously.
SAT modulo theory solvers have included theory solvers for difference constraints for some time.
While a theory solver for difference constraints gives the basis of a global difference constraint propagator, we show how the requirements on the propagator are quite different.
Crucially, we show how to explain propagations by a global difference constraint propagator, in order to use it within a lazy clause generation solver.
We give experiments showing that treating difference constraints globally can substantially improve on the standard propagation approach.

\paragraph*{Acknowledgments}
This research was partially funded by the Australian Government through the Australian Research Council Industrial Transformation Training Centre in Optimisation Technologies, Integrated Methodologies, and Applications (OPTIMA), Project ID IC200100009.
}


\maketitle

\section{Introduction}

Finite domain propagation is a powerful approach to solving complex combinatorial satisfaction and optimization problems.
The core of a finite domain propagation solver are the propagators: given a constraint $c$ and a domain $D$ representing the set of possible values of the variables in $c$, they remove the values of the variables which cannot take part in any solution subject to $c$.
The propagation solver interleaves the fix-point computation of the propagators with search to find a solution. 
In this paper, we study how we should propagate constraints of the form $x - y \leq d$.

Difference (logic) constraints (also called temporal constraints), that is $x - y \leq d$, where $x$ and $y$ are variables and $d$ is an integer constant, are one of the simplest forms of constraints.
They are well-studied, with efficient algorithms known for their satisfaction and implication, because of their connection to shortest path algorithms.

But traditionally, finite domain propagation solvers do not make use of these algorithms.
Difference constraints are represented as individual propagators just like any other constraints.
Unlike most constraints, propagation on difference constraints is complete.
That is, if $C$ is a set of difference constraints, then a propagation solver starting from domain $D$ will fail if $C \wedge D$ is unsatisfiable.

But the order of the propagation of the constraints can make a significant difference.
In a sense, the propagation engine implements a version of Fords labelling and scanning algorithm (see~\cite{handbook}, page 558) that also checks for negative cycles.
The potential bad behaviour of propagation on constraints of difference is well recognized.

\begin{example}\label{ex:worst}
    Consider the constraints $x - y \leq 0$ and $y - x \leq -2$.
    Together, these are unsatisfiable.
    If the initial domains of $x$ and $y$ are given by
    $x \geq 0$, $x \leq n$, $y \geq 0$, and $y \leq n$, 
    then a propagation engine will take $\cO(n)$ steps to determine the unsatisfiability of the constraints, 
    because in each step it will only remove the two highest values of one variable propagating one of these constraints.
    This could be determined in constant time.
    \hfill $\qed$
\end{example}

Half-reified constraints~\cite{cp2011a} of difference  $b \half x - y \leq d$, that is if $b$ then the constraint $x - y \leq d$ holds, are very useful in expressing scheduling and other problems.
Finite domain propagation engines typically implement them as weak bounds propagators that only propagate to either propagate the constraint $x - y \leq d$ when $b$ is known true or set $\lnot b$ if the bounds of $x$ and $y$ force it to be false.
Finite domain propagation engines are not complete for checking implication of half-reified difference constraints by difference constraints.

\begin{example}\label{ex:imp}
    Consider the constraints $c_1 \equiv y - x \leq -2$ and $c_2 \equiv x - z \leq 3$ with the initial domain
    $D = \{ x \geq 0, x \leq 10, y \geq 0, y \leq 10, z \geq 0, z \leq 10\}$.
    After propagation we learn $x \geq 2$ and $y \leq 8$.
    The half-reified constraint $b \rightarrow z - y \leq -2$ does not force $\lnot b$ although this is certainly a consequence of $c_1 \wedge c_2$ since $c_1 \wedge c_2 \models y - z\leq 1$ by simply summing the constraints.
    \hfill $\qed$
\end{example}

SAT Modulo Theories (SMT) solvers~\cite{smt, smt2} treat difference constraints quite differently than finite domain propagators.
Effectively, they treat fully reified difference constraints  $b \Leftrightarrow x - y \leq d$, equivalent to $\left(b \half x - y \leq d \right) \wedge \left(\neg b \half y - x \leq -d-1\right)$, where $b$ acts as the ``name'' of the constraint in the SMT solver.
They use specialized shortest path algorithms~\cite{Cotton:Maler:06} to determine the unsatisfiability of the difference constraints, which are consequences of the
current Boolean evaluation.
For example, if $b$ is true then $x - y \leq d$ and if $b$ is false, then $y - x \leq -d- 1$.
They also determine all the difference constraints that are entailed or disentailed by the difference constraints which are consequences of the current Boolean evaluation.

In this paper, we investigate how we can build a global propagator for difference constraints for use in a finite domain propagation engine, that treats all difference constraints together, that can determine all consequences of half-reified constraints.
As we shall see, the questions that arise in a finite domain propagation context are different from those in SMT, and there are different trade-offs for the implementation.
To use this propagator in a lazy clause generation solver~\cite{lazyj}, we need to be able to explain its propagators, just like in an SMT solver. 

An earlier much shorter version of this paper was originally published in PPDP2008~\cite{feydy2008global}.
The contributions of this paper are:
\begin{itemize}
    \item improved algorithms for bounds propagation using global difference logic,
    \item the development of hybrid propagation approaches that combine the global propagator with individual difference constraint propagators,
    \item the first implementation of a global difference logic propagator including explanations in a general purpose finite domain solver,
    \item the first evaluation of different explanation options, and
    \item experiments illustrating that treating difference logic constraint globally can be substantially more efficient than treating them as individual constraints.
\end{itemize}

The remainder of the paper is organized as follows.
\Cref{sec:fd} introduces our notation for finite domain solving, and illustrate the usual propagators for difference constraints.
\Cref{sec:diff} introduces notation for graphs and shortest paths, and gives the fundamental theorem relating difference constraint satisfaction and implication to shortest paths.
We then explain the state-of-the-art incremental algorithms for difference constraints.
Then, \Cref{sec:prop} defines the capabilities and default implementation of a global difference constraint propagator, before discussing how to improve it by handling bounds specially, and how to explain the resulting propagations. \Cref{sec:huub} shows the integration of the global difference logic propagator in a modern lazy clause generation solver including preprocessing and simplification.
Finally, \Cref{sec:exp} gives experimental results, \Cref{sec:related} discusses related work, and \Cref{sec:conc} concludes the paper.


\section{Finite Domain Solvers}
\label{sec:fd}

We now introduce our terminology for finite domain propagation based solving.
Let $\VV$ be a set of (integer) variables.
For theory purposes we will treat Boolean variables as 0-1 integers, but for clarity of exposition use $\bot = 0$ and $\top = 1$ to represent values taken by Boolean variables. 
In an abuse of notation $b$ and $\neg b$ for Boolean variable $b$ are treated as shorthand for the expressions $b = \top$ and $b = \bot$, respectively.
A \emph{valuation}, $\theta$, is a mapping of variables to values, denoted $\{x_1 \mapsto d_1, \ldots, x_n \mapsto d_n\}$.
We define $\vars(\theta) = \{x_1, \ldots, x_n\}$.
We can apply a valuation to a variable $\theta(x_i)$ to return the value $d_i$, and extend application of valuations $\theta$ to arbitrary expressions involving $\vars(\theta)$ in the obvious way.

A \emph{primitive constraint}, $c$, is a set of valuations over a set of variables $\vars(c)$.
A valuation $\theta$ is a \emph{solution} of $c$ if  $\{ x \mapsto \theta(x) ~|~ x \in vars(c) \} \in c$.
A \emph{constraint} $C$ is a conjunction of primitive constraints, which we often treat
as a set.
A valuation $\theta$ is a solution of constraint $C$ if it is a solution for each $c \in C$.
We write $C_1 \models C_2$ if every solution of $C_1$ is a solution of $C_2$.
We extend this notation to valuations, writing $\theta \models C$ if $\bigwedge_{i=1}^n x_i = d_i \models C$, where $\theta = \{x_1 \mapsto d_1, \ldots, x_n \mapsto d_n\}$. 

An \emph{atomic constraint} is a unary constraint (we can restrict to the forms $x = d, x \neq d, x \geq d, x \leq d$), or $\false$.
A \emph{domain} $D$ is a conjunction of atomic constraints over $vars(D)$.
$D$ is a \emph{false domain} if it has no solutions.  
We use notation $D(x) = \{ \theta(x) ~|~ \theta \text{~is a solution of~} D \}$.
We use \emph{range notation} $\range{l}{u} = \{ d ~|~ l \leq d \leq u\}$.
A domain $D$ is a \emph{range domain} if for each variable $x \in vars(D)$ then $\exists l,u: D(x) = \range{l}{u}$.
We will only be interested in range domains for the purpose of this paper.
We can map a valuation $\theta$ to a domain $D_0 = \wedge_{x \in vars(D)} x = \theta(x)$.

A \emph{propagator} $p(c)$ for constraint $c$ is an inference algorithm, it maps a domain $D$ to a conjunction of literals $p(c)(D)$, where $D \wedge c \models p(c)(D)$. 
We assume each propagator is \emph{checking}.
That is, if $\forall x \in vars(c). |D(x)| = 1$, then $p(c)(D) \wedge D \models \false$ if and only if $\theta_D$ is not a solution of $c$.
A \emph{domain propagator} $dom(c)$, for constraint $c$, is the strongest possible propagator for $c$:
after propagation, the domain $D' = D \wedge dom(c)(D)$ only contains values in $d \in D'(x), \forall x \in \VV$, where there is a solution $\theta$ of $D \wedge c$ and
$\theta(x) = d$.
A propagation solver $prop(\mathcal{P},D)$ applied to a set of propagators $\mathcal{P}$ and a domain $D$ repeatedly applies the propagators $p \in \mathcal{P}$ adding new atomic constraints from propagators to $D$ until it reaches a domain $D'$ where $D' \models D' \wedge p(D')$ for all $p \in \mathcal{P}$, and returns $D'$.
In an abuse of notation we write $prop(C,D)$ to refer to
$prop(\{ dom(c) \mid c \in C\}, D)$.

\begin{example}
The domain propagator $f \equiv dom(x - y \leq d)$ for the difference constraint $x - y \leq d$ can be implemented as:
\begin{equation*}
    f(D)  =  \{ x \leq d + \max D(y), y \geq \min D(x) - d \}.
\end{equation*}
A domain propagator $g = \dom(b \half x - y \leq d)$ can be implemented as follows.
\begin{align*}
g(D) & = f(D), &&\text{if}  \ D(b) = \{ \top \} \\
g(D) & = \{ \neg b \}, &&\text{if} \  \min D(x) - \max D(y) > d\\
g(D) & = \emptyset, &&\text{otherwise}
\end{align*}\hfill$\qed$
\end{example}

A \emph{constraint satisfaction problem (CSP)} is a constraint $C$, often broken into a domain constraint and the remainder $C \leftrightarrow D \wedge C'$. 

In \emph{lazy clause generation (LCG)} solvers~\cite{lazyj} propagators are also required to give explanations for each new consequence $l \in p(c)(D)$ (i.e. where $D \not\models l$).
That is, an explanation clause  $e \equiv l_1 \wedge \dots \wedge l_n \half l$, such that $\forall 1 \leq i \leq n, D \models l_i$ and $c \models e$. 
LCG solvers, like SAT solvers, create an implication graph, where every new consequence is attached to a reason.
On failure, this is used to create a \emph{nogood} by repeatedly replacing literals in the explanation of  failure until only one literal that became true after the last decision remains.
This nogood is guaranteed to generate new propagation information.
See~\cite{lazyj} for more details.

\begin{example}
Given domain 
$D = \{ x \geq 5, x \leq 10, y \geq 0, y \leq 10\}$ and
constraint $c \equiv x - y \leq 4$, then 
$dom(c)(D) = \{ x \leq 14, y \geq 1 \}$.
The first literal is redundant, the second literal needs to be explained. The explanation is simply
$x \geq 5 \rightarrow y \geq 1$.

Given the same domain and constraint $c' \equiv b \half x - y \leq -6$, then 
$dom(c')(D) = \{ b = 0\}$
and the explanation is
$x \geq 5 \wedge y \leq 10 \rightarrow b = 0$.
\hfill$\qed$
\end{example}

\section{Solving Difference Constraints}\label{sec:diff}

Difference constraints are highly connected to shortest paths, a statement we shall formalize shortly.
In this section we give our graph notation, and then explain how the difference constraints algorithms work.

\subsection{Graphs, Paths and Potential Functions}

A \emph{weighted directed graph} $G = (V,E)$ is made up of vertices $V$ and a set $E$ of weighted directed edges $(u, v, d)$ from vertex $u \in V$ to vertex $v \in V$ with weight $d$.
We also use the notation $u  \stackrel{d}{\to} v$ to denote the edge $(u,v,d)$.
A \emph{path} $P$ from $v_0$ to $v_k$ in graph $G$, denoted $v_0 \rightsquigarrow v_k$, is a sequence of edges $e_1, \ldots, e_k$ where $e_i = (v_{i-1}, v_i, d_i) \in E$.
A \emph{simple path} $P$ is a path where $v_i \neq v_j, 0 \leq i < j \leq k$.
A (simple) \emph{cycle} $P$ is a path $P$ where $v_0 = v_k$ and $v_i \neq v_j, 0 \leq i < j \leq k \wedge (i \neq 0 \vee j \neq k)$.
The \emph{path weight} of a path $P$, denoted $w(P)$ is $\Sigma_{i=1}^k d_i$.

Let $G$ be a graph without negative weight cycles, that is, without a cycle $P$ where $w(P) < 0$.
Then, we can define a \emph{shortest path} from $v_0$ to $v_k$, which we denote by $SP(v_0,v_k)$, as the (simple) path $P$ from $v_0$ to $v_k$ such that $w(P)$ is minimized.
Let $wSP(x,y) = w(SP(x,y))$ or $+\infty$ if no path exists from $x$ to $y$.
Given a graph $G$ and vertex $x$ define the functions $\delta_x^\leftarrow, \delta_x^\rightarrow: V\rightarrow \mathbb{R}$ as
\[
\delta_x^\leftarrow (y) = wSP(x,y)\qquad
 \text{and}\qquad
\delta_x^\rightarrow (y) = wSP(y,x)\enspace.
\]

Let $G$ be a graph without negative weight cycles.
Then $\pi$ is a \emph{valid potential function} for $G$ if $\pi(u) + d - \pi(v)\ge 0$ for every edge $(u,v,d)$ in $G$.

There are many algorithms (see e.g.~\cite{ncd}) for detecting negative weight cycles in a weighted directed graph.
They either detect a cycle or determine a valid potential function for the graph.
The standard approach is the Bellman-Ford algorithm, which is $\cO(|V||E|)$

Given a valid potential function $\pi$ for a graph $G=(V,E)$, we can define
the \emph{reduced cost graph} $rc(G)$ as $(V, \{ (x,y,\pi(x) + d -
\pi(y) ~|~ (x,y,d) \in E \})$.
All weights in the reduced cost graph are non-negative, and we can recover the original path length $w(P)$ for path $P$ from $x$ to $y$ from paths in the reduced cost graph.
This follows from $w(P) = w + \pi(y) - \pi(x)$, where $w$ is the weight of the corresponding path in the reduced cost graph.
Since edges in the reduced cost graph are non-negative, we can use Dijkstra's algorithm to calculate shortest paths in the reduced cost graph in time $\cO(|V|\log |V| + |E|)$ instead of $\cO(|V||E|)$ assuming the use of a Fibonacci heap.

\subsection{Difference Constraints}

Difference constraints are a well-studied class of constraints (e.g.~\cite{shostak,dechter1991temporal}).
Difference constraints are of the form $x - y \leq d$.
Note that we can encode bounds constraints of the form $x \geq l$ and $x \leq u$ by selecting a \emph{dummy} variable $v_0$ to represent the fixed value 0,
and encoding them as $v_0 - x \leq -l$ and $x - v_0 \leq u$ respectively.
We call these \emph{encoded bounds constraints}.
We can map difference constraints to a weighted directed graph. 
For the remainder of this paper, let $v_0$ denote the dummy variable.

\begin{definition}
    \label{def:constraintgraph}
   Let $C$ be a set of difference constraints and let $G_C=(V,E)$ be the graph comprised of one weighted edge $x \stackrel{d}{\to} y$ for every constraint $x-y\leq d$ in $C$.
   We call $G_C$ the \emph{constraint graph} of $C$.
\hfill $\qed$
\end{definition}

The following well-known result characterizes how the constraint graph can be used for satisfiability and implication checking of difference constraints.

\begin{theorem}\label{thm:diff}
    Let $C$ be a set of difference constraints and $G_C$ its corresponding graph.
    $C$ is satisfiable if and only if $G_C$ has no negative weight cycles, and if $C$ is satisfiable then $C \models x - y \leq d$ if and only if $wSP(x,y) \leq d$.\hfill $\qed$
\end{theorem}

Ramalingam~\emph{et al.}~\cite{ram} define efficient algorithms  for satisfiability of difference constraints after incrementally adding or deleting a constraint. 
Cotton and Maler~\cite{Cotton:Maler:06} define efficient incremental algorithms for difference constraints.
The satisfiability algorithm for incremental addition is identical to that of~\cite{ram}, while they also give an algorithm for checking implication of difference constraints on addition.
For our purposes, we are only interested in incremental addition algorithms so we will use the formulation of Cotton and Maler.

The incremental satisfaction algorithm for addition of~\cite{Cotton:Maler:06,ram}, shown in \Cref{alg:incsat}, relies on maintaining a potential function $\pi$ for the constraint graph $G_C$. 
In a sense, it is an incremental Bellman-Ford algorithm.
When a new constraint $u - v \leq d$ is added, the edge $u \stackrel{d}{\to} v$ is added to $G_C$ and a new potential function $\pi'$ is calculated or unsatisfiability
is detected.
The algorithm is $\cO(n \log n + m)$ for $m$ difference constraints on $n$ variables (using a Fibonacci heap to implement argmin).

\begin{algorithm2e}[tbp]
   \caption{\textsc{IncSat}~\cite{Cotton:Maler:06}}
   \label{alg:incsat}
   \SetKw{KwAll}{all}
   \SetKw{KwAnd}{and}
   \KwIn{$G_C = (V,E)$ a graph, $\pi$ a valid potential function for $G_C$,
   edge $(u,v,d)$ a new constraint to add to $G_C$.}
   \KwOut{UNSAT if $C \cup \{u - v \leq d\}$ is unsatisfiable, or
   $G_{C \cup \{u - v \leq d\}}$ and
   a valid potential function $\pi'$ for $G_{C \cup \{u - v \leq d\}}$.}
   $\gamma(v)$ := $\pi(u) + d - \pi(v)$\;
   $\gamma(w)$ := 0 for all $w \neq v$\;
   $\pi'(v)$ := $\pi(v)$ for all $v \in {\cal V}$\;
   \While{ $min(\gamma) < 0 \wedge \gamma(u) = 0$}{
     $s$ := argmin($\gamma$)\;
     $\pi'(s)$ := $\pi(s) + \gamma(s)$\;
     $\gamma(s)$ := 0\;
     \For{\KwAll $s \stackrel{d'}{\rightarrow} t \in G$}{
       \If{$\pi'(t) = \pi(t)$}{$\gamma(t)$ := $\min \{ \gamma(t), \pi'(s) + d'
         - \pi'(t)\}$}
     }
   }
   \If{$\gamma(u) < 0$}{\Return{UNSAT}}
   \Return{$((V, E \cup \{(u,v,d)\}), \pi')$}
\end{algorithm2e}

The implication algorithm of \cite{Cotton:Maler:06}, shown in \Cref{alg:incimp}, simply checks, for each difference constraint $x - y \leq d$ of interest, whether the new edge has created a path from $x$ to $y$ of length $\leq d$.
It makes use of the potential function previously calculated to compute shortest paths on the reduced cost graph using Dijkstra, rather than using a more expensive algorithm that handles negative weight edges.
The algorithm is $\cO(n \log n + m + p)$ for $m$ difference constraints on $n$ variables and $p$ constraints to check for implication.

\begin{algorithm2e}[htbp]
    \caption{\textsc{IncImp}~\cite{Cotton:Maler:06}}
    \label{alg:incimp}
    \KwIn{
        $G=(V,E)$ a \emph{constraint graph} representing a set of difference
        constraints $C \cup \{u - v \leq  d\}$, $\pi$ a \emph{valid potential
        function} on $G$, a set of difference constraints $C'$ where 
        $C \not\models c', \forall c' \in C'$.
    }
    \KwOut{The set $C'' \subseteq C'$ of constraints not implied by
    $C \cup \{u - v \leq d\}$.}
    \SetKw{KwAll}{all}
    \SetKw{KwAnd}{and}
    \SetKw{KwOr}{or}
        $C''$ := $\emptyset$\;
        compute $\delta^\leftarrow_u$ and $\delta^\rightarrow_v$ by using $rc(G)$
        via $\pi$\;
        \DontPrintSemicolon
        \For{\KwAll $c'= (x - y \leq d') \in C'$}{
            \lIf{
                $\delta^\leftarrow_u(x) + d +\delta^\to_v(y) > d'$
            }{
                $C''$ := $C'' \cup \{c'\}$;
            }
        }
        \PrintSemicolon
   \Return{$C''$}
\end{algorithm2e}

\begin{example}\label{ex:cm}
Consider the system of constraints $C = \{x - y \leq - 2, y - z \leq 3\}$, where $D(x) = D(y) = D(z) = \range{0}{10}$.
The corresponding graph $G_C$ is shown in \Cref{fig:ex}a.
Given a valid potential function $\pi(v_0) = 0$, $\pi(x) = -3$, $\pi(y) = -8$, $\pi(z) = -7$, the reduced cost graph of $G_C$ is shown in \Cref{fig:ex}b.

Consider the addition of the constraint $y - x \leq 0$ using \textsc{IncSat}.
$\gamma(x) = -5$ and $\gamma$ for the remaining variables is 0.
$x$ is the minimal $\gamma$ value.
We set $\pi'(x) = -8$ and set $\gamma(x) = 0$.
We then adjust the $\gamma$ values: $\gamma(y) = -2$.
Now $\gamma(y) \neq 0$ and the loop terminates and returns \textit{UNSAT}.
The unsatisfiable loop has been found just examining the nodes $x$ and $y$ and their outgoing edges.

Now consider the implication test for the constraint $x - z \leq 4$ using \textsc{IncImp}.
Assuming we have just added the constraint $x - y \leq -2$ and obtained the potential function illustrated in \Cref{fig:ex}b, we can compute $\delta_x^\leftarrow$ from this reduced cost graph.
First, we compute the weights of shortest paths to $x$ $wSP(v_0,x) = 3$, $wSP(z,x) = 6$, $wSP(y,x) = 5$, and then calculate $\delta_x^\leftarrow$ using the potential values $\delta_x^\leftarrow(x) = 0$, $\delta_x^\leftarrow(v_0) = 0$, $\delta_x^\leftarrow(z) = 10$, and $\delta_x^\leftarrow(y) = 10$.
We similarly calculate $wSP(y,z) = 2$, $wSP(y,v_0) = 2$, $wSP(y,x) = 5$ in the reduced cost graph and hence $\delta_y^\rightarrow(y) = 0$, $\delta_y^\rightarrow(z) = 3$, $\delta_y^\rightarrow(v_0) = 10$, and $\delta_y^\rightarrow(x) = 10$.
Examining $\delta_x^\leftarrow(x) + -2 + \delta_y^\rightarrow(z) = 0 + -2 + 3$, we find this is less that 4 and hence the constraint is implied.
\hfill $\qed$

\begin{figure}[tbp]
    \centering
\begin{tabular}{cc}
(a) & $
\xymatrix{
v_0 \ar[r]_0 \ar@/_5mm/[rr]_0 \ar@/^5mm/[rrr]^0 &
x \ar[r]_{-2} \ar@/^5mm/[rrr]^{10} &
y \ar[r]_{3} \ar@/_5mm/[rr]_{10} &
z \ar[r]_{10} & v_0
}
$ 
\\
(b) &
$
\xymatrix{
\stackrel{0}{v_0} \ar[r]_3 \ar@/_5mm/[rr]_8 \ar@/^5mm/[rrr]^7 &
\stackrel{-3}{x} \ar[r]_{3} \ar@/^5mm/[rrr]^{7} &
\stackrel{-8}{y} \ar[r]_{2} \ar@/_5mm/[rr]_{2} &
\stackrel{-7}{z} \ar[r]_{3} & \stackrel{0}{v_0}
}
$
\end{tabular}
\caption{The corresponding constraint graph $G_C$ for a system of difference
  constraints $C = \{x - y \leq - 2, y - z \leq 3\}$ is shown in (a).
The node for dummy variable $v_0$ is shown twice for clarity of presentation.
The reduced cost graph is shown in (b) assuming the potential value shown
above the nodes. Note how all edge lengths are non-negative. 
}
 \label{fig:ex}
\end{figure}
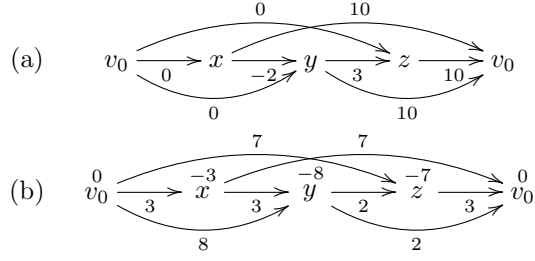
\end{example}

\section{Difference Constraint Propagation}\label{sec:prop}

In this section, we explain how we can create a global propagator for all difference constraints and reified difference constraints.

\subsection{A Global Difference Constraint Propagator}\label{sec:base}

We can use \textsc{IncSat} and \textsc{IncImp} to create a global propagator as follows. The global propagator needs to support the following operations:
\begin{itemize}
\item Add a new difference constraints $x - y \leq d$
\item Add a new bound $x \geq l$ or $x \leq u$
\item Add a new half-reified difference constraint  $b
\half u - v \leq d'$
\item Mark a state and backtrack to a previous marked state
\end{itemize}

Adding a new difference constraint  $x - y \leq d$ causes the propagator to run \textsc{IncSat}, and fails if this determines unsatisfiability.
If it does not, then it runs \textsc{IncImp} to determine if any difference constraints $x - y \leq d' \in C'$ become entailed.
For each half-reified difference logic constraint $b \Rightarrow y - x \leq -d'-1$, if we find that $x - y \leq d'$ is entailed, then we set $b=0$.
This is the same as an SMT difference constraint solver modulo the use of half reification.

This is not enough, since we are also interested in any new bounds that result from the constraint addition.
A corollary of \Cref{thm:diff} gives the key insight.

\begin{corollary}
If $C$ is a satisfiable set of difference and (encoded) bounds constraint
then $C \models z \geq l$ if and only if $wSP(v_0,z) \leq -l$ and
$C \models z \leq u$ if and only if $wSP(z,v_0) \leq u$. \hfill $\qed$
\end{corollary}

To extract new bounds, the propagator needs to determine for each variable $z$ the weight of a shortest path from $v_0$ to $z$ via the new edge $(x,y,d)$, i.e. $\delta^\leftarrow_x(v_0) + d +\delta^\to_y(z)$. 
The negation of this is a lower bound on $z$.
It also determines the weight of a shortest path from $z$ to $v_0$ via the new edge
$(x,y,d)$, $\delta^\leftarrow_x(z) + d +\delta^\to_y(v_0)$, which is an
upper bound on $z$.
The propagator updates the bounds of variables which have changed regarding these weights.

\begin{example}
Consider adding the new constraint $x - y \leq -2$ to obtain the system of constraints $C$ of \Cref{ex:cm} illustrated in \Cref{fig:ex}a. 
To extract new bounds, we need to determine $\delta^\leftarrow_x(v_0)$ but this is just $ - \min D(x)$, we have already determined $\delta^\to_y$ for \textsc{IncImp}.
The possibly new bounds are $0 + -2 + 0$ for $y$ (or a lower bound of 2), $0 + -2 + 3$ for $z$ (or a lower bound of -1, so not new).
The upper bounds are calculated similarly as $0 + -2 + 10$ for $x$ or 8.
The resulting domain is $D(x) = \range{0}{8}$, $D(y) = \range{2}{10}$, and $D(z) = \range{0}{10}$.
\hfill $\qed$
\end{example}

Adding a new bound $x \leq u$ or $x \geq l$ is simply adding a new
difference constraint $x - v_0 \leq u$ or $v_0 - x \leq -l$.
This is then treated as described in the previous paragraphs.

Adding a new half-reified constraint $b \Rightarrow u - v \leq d'$ is a new feature not usually present in SMT solvers, since in this context all reified constraints are usually known from the beginning of solving.
We need to check if the constraint is disentailed with the current set of difference constraints.
To do so, we calculate $wSP(u,v)$ and $wSP(v,u)$ by using Dijkstra's algorithm on the reduced cost graph.
Note that half-reified bounds constraints $b \half u \leq d'$ or $b \half -v \leq d'$ can be handled by examining the domain $D$ directly.

For backtracking, we need to restore the state of the difference constraint propagator to that at an earlier time.
The only state of the solver is the set of difference constraints posted, and the set of posted half-reified difference constraints, as well as the valid potential function.
The potential function $\pi$ remains valid on backtracking since it simply ensures that $\pi(x) + d - \pi(y) \geq 0$ for all edges $(x,y,d) \in G$, and removing edges does not invalidate this.
Hence, the potential function need not be trailed, as noted by Wang~et~al.~\cite{Wang:Ivancic:Ganai:Gupta:05}.
Backtracking must simply remove the representation of the constraints added since the marked state.

Under the assumptions that the domain $D$ is a range domain and no Boolean variable appears twice in the set of difference constraints, bounds constraints, and (added) half-reified difference constraints $C$, the global propagator defined in this subsection is a domain propagator for (the conjunction) $C$.


\subsection{Handling Bounds Constraints Better}\label{sec:bnds}

While in the SMT context bounds are simply another form of difference constraint, for a propagation engine bounds updates are much more frequent than difference constraint additions.
Hence, it is worth treating them separately.
We will not use the dummy variable $v_0$ to encode bounds constraints but treat them directly.
The basis of this treatment is the following theorem.

\begin{theorem}
\label{thm:imp}
Let $C$ be a set of difference constraints and $D$ a range domain,
$c$ a difference constraint, and $D' = prop(C,D)$ the domain after propagation.
Then (a) $D'$ is a false domain if and only if $D \wedge C$ is unsatisfiable and (b) $C \wedge D \models c$ if and only if $C \models c$ or $D' \models c$.
This means we can check implications by checking implication by difference constraints alone, and implication by bounds alone.
\end{theorem}
\begin{proof}
(a) ($\leftarrow$) By the correctness of propagation $C \wedge D \models prop(C,D) = D'$ and hence if $D'$ is a false domain $C \wedge D$ is unsatisfiable.
($\rightarrow$) Let $G$ be the graph encoding $C$ and $D$ (as encoded bounds
constraints using dummy variable $v_0$), then by \Cref{thm:diff} we have that the
corresponding graph has a negative weight cycle $p$.
W.l.o.g., let $p$ be from $v \rightsquigarrow v$ with weight $w<0$, $D(v) = \range{l_v}{u_v}$ be the range of $v$ and $k$ be chosen so that $k w + u_v - l_v < 0$.

Then consider the following path.
$$
v_0 \stackrel{-l_v}{\longrightarrow} \stackrel{k ~\text{times}}{\overbrace{v
\rightsquigarrow v
\cdots v
  \rightsquigarrow v}} \stackrel{u_v}{\longrightarrow} v_0
$$
Now standard bounds propagation, by choosing the constraints in the order of this path, eventually sets the lower bound of $v$ to $l_v - k w > u_v$ and creates a false domain.
Since any order of propagating constraints leads to the same result, $D'$ is a false domain.

(b)($\leftarrow$) Clearly $C \models c$ or $D' \models c$ imply that $C \wedge D \models c$ since $C \wedge D \models D'$.
($\rightarrow$) By \Cref{thm:diff}, we have that $c \equiv x - y \leq d$ is implied by the graph $G$ encoding $C$ and $D$ using encoded bounds constraints if and only if a shortest path from $x$ to $y$ is length less than or equal to $d$.
Suppose the shortest path does not visit $v_0$.
Then, clearly $C \models c$ since $G_C$ (without the encoded bound constraints) has the same shortest path.
Otherwise, $SP(x,y) = x \rightsquigarrow v_0 \rightsquigarrow y$, where $w(x \rightsquigarrow v_0) = w_1$, $w(v_0 \rightsquigarrow y) = w_2$, and $w_1 + w_2 \leq d$.
Bounds propagation on the path $v_0 \rightsquigarrow y$ sets the lower bound of $y$ to at least $-w_2$.
Bounds propagation on the path $x \rightsquigarrow v_0$ sets the upper bound of $x$ to at most $w_1$.
Hence, $D' \models x \leq w_1 \wedge y \geq -w_2$ and hence $D' \models x - y \leq d$.
\end{proof}

The above theorem implies we can check the satisfiability of bounds constraints using propagation on the difference constraints, and we can split the implication into two checks: just using the difference constraints and just using the computed bounds.

At first, let us consider bounds updating.
We define an algorithm that simultaneously considers all bounds changes since the last time the global propagator was run, as opposed to adding them one by one, which is required by the base approach.

\begin{algorithm2e}[htb]
   \caption{\textsc{IncLB}}
   \label{alg:BoundCheck}
   \KwIn{$G_C=(\VV,E)$ a \emph{constraint graph} representing set of difference
     constraints $C$, $\pi$ a \emph{valid potential function} on $G_C$,
     a range domain $D_o$ giving the upper and lower bounds of
     variables the last time the propagator was run, and a range
     domain $D$ giving the current bounds.}
   \KwOut{A set of lower bounds constraints $B$ giving new bounds for variables
   in $\VV$.}
   \SetKw{KwAll}{all}
   \SetKw{KwAnd}{and}
   \SetKw{KwOr}{or}
   $\VV_l := \{x\in\VV\mid \min D(x) > \min {D_o}(x)\}$\;
\ignore{
   \tcp{Start of the computation of the new lower bounds}
}
   $\pi(v_0) := \max\{\min D(x) + \pi(x) \mid x\in \VV_l \}$\;
   \% Dijktras algorithm from $v_0$ on the reduced cost graph $rc(G)$ 
   augmented so
   that variable $s$ is not considered if the new lower bound
   $-\delta^\to_{v_0}(s)$ is smaller than the existing bound $\min D(s)$\;
   $\gamma(v)$ := $\pi(v_0) - \min D(v) - \pi(v)$ for all $v \in \VV_l$\;
   $\gamma(v)$ := $+\infty$ for all $v \in \VV \setminus \VV_l$\;
   $wSP(v_0,v)$ := $+\infty$ for all $v \in \VV$\;
   \While{$\min(\gamma) < +\infty$}{
     $s$ := argmin($\gamma$) \;
     $wSP(v_0,s)$ := $\gamma(s)$\;
     $\gamma(s)$ := $+\infty$ \;
     $\delta^\to_{v_0}(s)$ := $wSP(v_0,s) + \pi(s) - \pi(v_0)$\; 
     \If{$- \delta^\to_{v_0}(s) > \min D_o(s)$}{
     \For{\KwAll $s \stackrel{d'}{\rightarrow} t \in G$}{
       \If{$wSP(v_0,t) = +\infty$}
        {\If{$\gamma(s) + \pi(s) + d' - \pi(t) < \gamma(t)$}
          {$\gamma(t)$ := $\gamma(s) + \pi(s) + d' - \pi(t)$\;
          }
     }
   }
   }
   }
   \% Convert new bounds into constraints\;
   \For{\KwAll $v\in \VV$}{
      \If{$-\delta^\rightarrow_{v_0}(v) > \min D(v)$}{
        $B$ := $B \cup \{ -v \leq \delta^\rightarrow_{v_0}\}$  
      }
   }
\ignore{
   \tcp{Start of the computation of the new upper bounds}
   let $\VV_u := \{x\in\VV\mid \max D(x) < \sup_{D_o}(x)\}$\;
   $\pi(v_0) := \min\{\max D(x) + \pi(x)\mid x\in \VV_u\}$\;
   compute $\delta^\leftarrow_{v_0}$ with starting queue $\VV_u$ and their
   values
   $\pi(x) + \max D(x) - \pi(v_0)$ ($x\in \VV_u$) by using the reversed graph of
   $rc(G)$ via $\pi$. Again don't add variable $v$ to the queue if
   $\delta^\leftarrow_{v_0}(v) > \max D(v)$ since there is no new bound\;
   \For{\KwAll $x\in \VV$}{
      \lIf{$\max D(x) > \delta^\leftarrow_{v_0}(x)$}{
         $B := B\cup\{x \leq \delta^\leftarrow_{v_0}(x)\}$
      }
   }
}
   \Return{B}
\end{algorithm2e}


Since the calculation of the new lower and upper bounds is symmetric, we
describe only the algorithm \textsc{IncLB}
(see \Cref{alg:BoundCheck}) for the lower bounds.

For the new lower bounds of a variable $x$, we want to know if the negative of the weight of the shortest path to $x$ from $v_0$, $- \delta^\to_{v_0}(x)$ is greater than the current lower bound of $x$, $\min D(x)$.
Because we do not represent a dummy variable $v_0$ in the constraint graph, we compute a valid potential function value $\pi(v_0)$ for $v_0$ regarding $\pi$ in the first step.
For this calculation, we consider only variables whose current lower bound is greater than their lower bound resulting from the last run of \textsc{IncLB}.

At the second step, we run Dijkstra on the reduced cost graph with a starting priority queue of the variables which have changed lower bound since the last run of the propagator.
The initial value for variable $x$ is the reduced cost of the ``imaginary'' edge between $v_0$ and $x$.
At last, we create the bounds constraints for calculated new bounds.

Our Dijkstra's algorithm does not explore all variables in the graph, it visits only variables $x$ for which $-\delta^\to_{v_0}(x) > \min D(x)$ holds, that is where a new lower bound has been found.

\begin{example}\label{ex:b}
Consider the set of constraints $C = \{$ $x - y \leq -2$, $y - z \leq 3$, $z - u \leq -1$, $u - v \leq 2$, $x - t \leq 1$, $t - z \leq -1$$\}$.
The constraint graph $G_C$ is shown in \Cref{fig:b}a. 
A valid potential function for $G_C$ is $\pi(x) = -3$, $\pi(y) = -8$, $\pi(z) = -7$, $\pi(u) = -9$, $\pi(v) = -7$, $\pi(t) = -4$.
In fact, since the graph does not include the dummy variable $v_0$, there are no cycles.
The domain $D(x) = \range{0}{18}$, $D(y) = \range{2}{20}$, $D(z) = \range{6}{19}$, $D(u) = \range{8}{20}$, $D(v) = \range{11}{20}$, $D(t) = \range{0}{18}$ is a fix-point for propagation on these constraints.

Suppose we update the lower bounds of variables $t$ to $1$ and $x$ to $5$. 
The algorithm then works as follows: $\VV_l = \{t,x\}$ and we compute $\pi(v_0) = 2$. 
Effectively, we will be searching for shortest paths from $v_0$ from the reduced cost graph shown in \Cref{fig:b}b.

Dijktra's algorithm determines $wSP(v_0,x) = 0$ or equivalently $\delta_{v_0}^\rightarrow(x) = 0 + -3 - 2 = -5$ and the lower bound of $x$ is $5$, a tighter bound, so the algorithm updates values for edges leaving $x$.
Then, $wSP(v_0,t) = 2$ or $\delta_{v_0}^\rightarrow(t) = 2 + -4 - 2 = -4$. 
This is again a new bound of $4$
and $t$'s neighbours are enqueued.
Then, $wSP(v_0,y) = 3$ or $\delta_{v_0}^\rightarrow(y) = 3 + -8 - 2 = -7$, which is again a new lower bound.
Then, $wSP(v_0,z) = 4$ and $\delta_{v_0}^\rightarrow(z) = 4 + -7 - 2 = -5$, and another new bound
$-z \leq -5$. 
When we calculate $wSP(v_0,u) = 5$ and  $\delta_{v_0}^\rightarrow(u) = 5 + -9 - 2 = -6$,  the new bound $6$ is not stronger than the existing bound, so no new propagation occurs.
The algorithm never visits node $v$.
We return the new bounds 
$B$ = $\{ t \geq 4$, $y \geq 7$, $z \geq 5\}$.

Compare this with the naive global difference approach of the previous section.
Each new bound is a new constraint  $v_0 - t \leq 1$ and $v_0 - x \leq -5$, and the graph shown in \Cref{fig:b}a would have 12 additional edges to and from the dummy node $v_0$. 
Adding the first constraint may determine a new potential function, and then shortest paths from $v_0$ are determined for each variable, and bounds updated.
Then the second constraint is added and a possibly new potential function computed and once more the shortest paths from $v_0$ are determined and bounds updated.
The improved method has a smaller graph, does not update potential function values, and visits each edge at most once regardless of the number of bounds changes since the last execution.
\hfill $\qed$

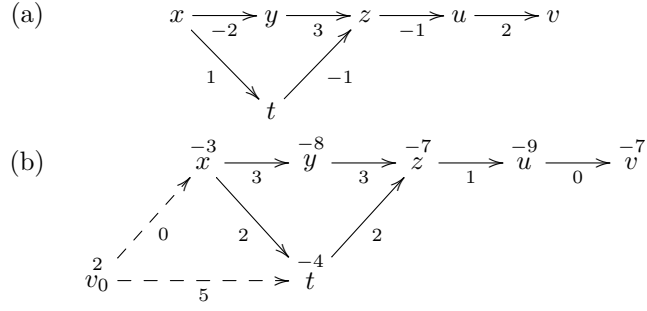
\begin{figure}[t]
\centering
\begin{tabular}{cc}
(a) & $
\xymatrix{
x \ar[r]_{-2}\ar[dr]_{1} &
y \ar[r]_{3}  &
z \ar[r]_{-1} &
u \ar[r]_2 & 
v \\
& t \ar[ur]_{-1}  
}
$ 
\\
(b) &
$
\xymatrix{
& \stackrel{-3}{x} \ar[r]_{3}\ar[dr]_{2} &
\stackrel{-8}{y} \ar[r]_{3}  &
\stackrel{-7}{z} \ar[r]_{1} &
\stackrel{-9}{u} \ar[r]_0 & 
\stackrel{-7}{v} \\
\stackrel{2}{v_0} \ar@{-->}[ur]_0\ar@{-->}[rr]_5 && \stackrel{-4}{t} \ar[ur]_{2}  
}
$ 
\end{tabular}
\caption{(a) The corresponding constraint graph $G_C$ for a system of difference
  constraints of \Cref{ex:b} and (b) the 
reduced cost graph with the dummy node $v_0$ and imaginary edges (dashed) added.
}
 \label{fig:b}
\end{figure}
\end{example}

Note that \textsc{IncLB} (and \textsc{IncUB}) requires $\cO(n \log n + m)$ time and $\cO(n + m)$ space for $m$ difference constraints (not bounds constraints) on $n$ variables.
Incremental propagation of difference constraints using a FIFO queue of propagators\footnote{A FIFO queue is the usual case, and for a LIFO queue it is easy to create even worse behavior.} can require $\cO(nm)$ time.

\begin{figure*}[t]
$$
\xymatrix{
y_0 \ar[r]\ar@/_5mm/[rr]_{1}\ar@/^6mm/[rrr]_{2}
\ar@/^9mm/[rrrrr]_{n-2}
\ar@/^12mm/[rrrrrr]^{n-1} & y_1 \ar[r] & y_2 \ar[r] & y_3 \ar[r] & \ldots \ar[r] &
y_{n-1}
\ar[r] & y_n \ar[d] \\
x_n & x_{n-1}\ar[l]  & x_{n-2}\ar[l]\ar@/^3mm/[ll]  &
x_{n-3}\ar[l]\ar@/^3mm/[ll] \ar@/^6mm/[lll]
  &
\ldots\ar[l]  & x_1\ar[l] \ar@/^3mm/[ll] \ar@/^6mm/[lll]
\ar@/^9mm/[llll]\ar@/^12mm/[lllll]  & x_0 \ar[l] \ar@/^3mm/[lll]
\ar@/^6mm/[llll]
\ar@/^9mm/[lllll]\ar@/^12mm/[llllll]
\\
\\
}
$$
\caption{The corresponding constraint graph for a system of difference constraints of \Cref{ex:n3}. Edges with 0 weight are unlabelled.}
 \label{fig:g}
\end{figure*}
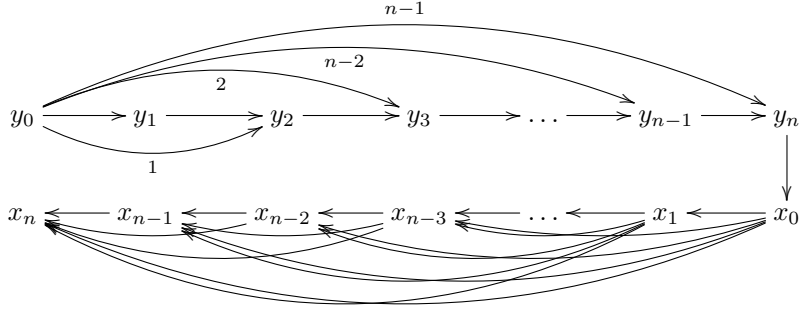

\begin{example}\label{ex:n3}
Consider the system of difference constraints $C$ defined as
$y_{i-1} - y_i \leq 0, 2 \leq i \leq n$,
$y_0 - y_i \leq i - 1, 1 \leq i \leq n$ and
$y_n - x_0 \leq 0$,
$x_i - x_j \leq 0, 0 \leq i < j \leq n$
illustrated in \Cref{fig:g}.
Domain $D$ is a fixpoint for $C$ when 
$D(y_i) = \range{0}{kn}, 0 \leq i \leq n$ and
$D(x_i) = \range{0}{kn}, 0 \leq i \leq n$ for a $k\geq1$.
None of the difference constraints is implied by the domain, which would mean
it could be removed from the propagation engine.

Consider when the domain of $y_0$ becomes $\range{n}{kn}$.
All constraints involving $y_0$ are queued for propagation.
If we are unlucky we will first bounds propagate on
$y_0 - y_n \leq n-1$, which will change $D(y_n)$ to $\range{1}{kn}$
and queue the $y_n - x_0 \leq 0$. We then propagate on
$y_0 - y_{n-1} \leq n-2$ which changes $D(y_{n-1})$ to $\range{2}{kn}$
and queues $y_{n-1} - y_n \leq 0$. Continuing we modify each domain
$D(y_i), 1 \leq i \leq n$
to $\range{n - i + 1}{kn}$ and queue each $y_{i-1} - y_i \leq 0, 2 \leq
i \leq n$ in reverse order.
The next propagator considered is $y_n - x_0 \leq 0$ which changes
$D(x_0)$ to $\range{1}{kn}$ and queues the constraints on $x_0$.
Then we consider $y_{n-1} - y_n \leq 0$ and modify $D(y_n) = \range{2}{kn}$
and queue $y_n - x_0 \leq 0$ again.
Continuing we modify each domain
$D(y_i), 2 \leq i \leq n$
to $\range{n - i + 2}{kn}$ and queue each $y_{i-1} - y_i \leq 0, 3 \leq
i \leq n$ in reverse order.
We then propagate the constraints on $x_0$ setting all $D(x_i)$ to
$\range{1}{kn}$ and queuing all constraints on $x_i, 1 \leq i \leq n$.
We then propagate on $y_{n} - x_n \leq 0$ again
and modify $D(x_0) = \range{2}{kn}$. We queue the constraints
on $x_0$ once more.  To reach a fixpoint the domain of $x_0$ changes
$n$ times, and we queue each constraint on $x$ $n$ times, hence the process
is $\cO(n^3)$.

Compare this with the execution of \textsc{IncLB}. Since all the arcs are
positive we can assume $\pi(v) = 0, \forall v \in V$ and the reduced cost
graph is identical to the original graph.
The algorithm will process each edge exactly once, visiting the nodes in
order $y_0$, $y_1$, \ldots, $y_n$, $x_0$ followed by the remainder of
the $x$ nodes in any order. The process is $\cO(n^2)$.

Finally note the whole process of increasing the lower bound of $y_0$ by
$n$ can be carried out $k-1$ times down a branch of the search tree!
\hfill $\qed$
\end{example}

We can thus check satisfiability of addition of bounds constraints
by running \textsc{\mbox{IncLB}} on the new lower bounds
and \textsc{IncUB} on the new upper bounds and seeing if in the result
any variable has an empty domain.
Afterwards, we can check new implications
of difference constraints caused by the addition of bounds constraints
by simply checking for each $(x - y \leq d) \in C'$ if it is implied
by bounds, i.e. $\max D(x) - \min D(y) \leq d$.

We check satisfiability on addition of a new difference constraints
$u - v \leq d$ by running \textsc{IncSat} on the graph $G_C$ containing
only the difference constraints (not bounds constraints), and then
perform bounds propagation by determining the possibly new lower bound
for $v$ given by $-d + \min D(u)$ and doing bounds propagation with
\textsc{IncLB} and the possibly new upper bound for $u$ of $d + \max D(v)$
and doing bounds propagation with \textsc{IncUB}.
We can then check implication by first checking if the bounds imply
some difference constraint in $C'$, and then running  \textsc{IncImp}
on the graph $G_{C\cup \{u -v \leq d\}}$.

Adding a new half-reified difference constraint $b \half u - v \leq d'$
requires us to check if it is disentailed by the domain
$D$, and if not we calculate $wSP(u,v)$ and $wSP(v,u)$ using the
reduced cost graph of $G_C$.


\subsection{Explanations}

In an LCG solver, each propagation needs to be justified by an explanation
that gives the reason for the propagation in clausal form.
Difference logic relies on three different kinds of propagations.

For the simplicity of explaining, we extend the notion of an edge~$s  \stackrel{d}{\to} t$
in the constraint graph $G_C$ with a Boolean variable~$b$ and depict it with $s\impl bd t$.
In the case the edge originates from a different constraint $s - t \leq d$ then $b = \top$.
In the other case, $b$ is the same Boolean variable as from the corresponding half-reified
constraint represented by the edge.
In addition, we denote $\mathit{pred}(t)=s$ the predecessor of~$t$ on the shortest path to~$t$ and $\mathit{pred}^\rightarrow(t)=b$
the Boolean variable of the edge $\mathit{pred}(t)\impl bd t$,
both can be determined in line 16 of \Cref{alg:BoundCheck}, when a new lower bound of $t$ is detected.

\subsubsection{Explaining Bound Propagations}

The simplest forms of propagation are bound propagations in \textsc{IncLB} and \textsc{IncUB}.
While the algorithms explore bound propagations jointly for the whole constraint graph~$G_C = (V, E)$,
each individual bound can be explained locally by the arc that caused the bound change.
Note this explanation is more reusable than more complex explanations.
For lower bound changes in \textsc{IncLB}, if $\gamma(t)$ is updated in line 16 of \Cref{alg:BoundCheck} along edge $s\impl bd t$, the explanation for the lower bound update of $t$ is that the edge is active together with the bound of the predecessor:
\begin{equation}
    \mathit{pred}(t)\geq-\delta_{v_0}^\rightarrow(\mathit{pred}(t))\land \mathit{pred}(t)^\rightarrow \half t\geq-\delta_{v_0}^\rightarrow(t))
\end{equation}

The upper bound updates are explain analogously following edges in opposite direction.

\subsubsection{Explaining Cycles of Negative Length}

The other two types of propagations both propagate Boolean variables used in implications to $\bot$. The simpler case is that a cycle of negative length is found at the addition of a new edge in \textsc{IncSat}. This either triggers a conflict if it is the consequence of adding a new edge $u\impl bd v$ with $b\mapsto\top$ or it is applied to half-reified edges after execution of \textsc{IncImp} to propagate and explain setting $b\mapsto\bot$. Similar to the previous case, the predecessor $\mathit{pred}(t)=s$ and the Boolean implying the preceding edge $\mathit{pred}^\rightarrow(t)=b$ is stored in line 10 of \Cref{alg:incsat}. $\gamma(u)<0$ in line 11 entails that there is a path of negative length from $u$ back to $u$, so the chain of predecessors in this loop is unrolled for the explanation:
\begin{equation}
    \mathit{pred}^\rightarrow(u)\land\mathit{pred}^\rightarrow(\mathit{pred}(u))\land\ldots \half \neg b
\end{equation}

The edge under consideration from $u$ to $v$ closing the cycle will never be added to the predecessor relation, and no other cycles of negative length exist when applying \textsc{IncSat}, so the explanation chain is guaranteed to end at $v$ which forms the root of the predecessor tree. Interestingly, Cotton and Maler~\cite{Cotton:Maler:06} who define \Cref{alg:incsat} do not really describe how to explain negative cycles, though they use this in their implementation. 

\subsubsection{Explaining Boolean Propagations Based on Bounds}

Finally, there is a third type of propagation that sets Boolean variables $b$ of disentailed constraints $b \half x - y \leq d'$ where the constraint is disentailed 
by the bounds on $x$ and $y$, so $\min D(x) - \max D(y) > d'$. 
The direct explanation (\texttt{Simple}) for this propagation is the following:
\begin{equation}
\label{eq:expl_simple}
    x \geq\min D(x)\land y\leq\max D(y) \half \neg b
\end{equation}

However, this explanation is not necessarily as strong as it could be. The propagation occurs as soon as $\min D(x)-\max D(y)>d'$, while the explanation does not consider how much above $d'$ the difference actually is. 
Therefore, the explanation can be lifted (\texttt{Lifted}):
\begin{equation}
    \mathit{expl}(\neg b)=\begin{cases}
        x\geq\max D(y)+d'+1\land y\leq\max D(y) \half \neg b & \text{if }\min D(x) \text{ was updated}\\
        x\geq\min D(x)\land y\leq\min D(y)-d'-1 \half \neg b & \text{otherwise}
    \end{cases}
\end{equation}
Now the most recently moved bound (or the lower bound if both were updated) is included with the least restrictive value that causes propagation of the Boolean variable.

Finally, a third option under consideration is to produce the lifted explanation lazily (\texttt{Lazy}), i.e., only when actually needed, and based on literals already in use by the solver (so lifting might occur on both bounds ensuring that the combination of both bounds is still a valid explanation).

\subsection{Possible Variations and Optimizations}
\label{sec:optimizations}

\Cref{thm:imp} also provides us with other ways of building
difference constraint propagators by mixing the standard approach
of a propagator per difference constraint and reified difference
constraint with the global approach.

Standard propagators will perform bounds propagation and implication by
bounds whereas the global propagator
performs a consistency check of the
system of difference constraints.
Hence we can combine the global propagator which only considers
addition of difference constraints $x - y \leq d$ and
half-reified difference constraints $b \half u - v \leq d'$
with the usual propagators for these constraints.
By prioritizing the global propagator before the standard propagators
we can avoid the worst case behavior of \Cref{ex:worst}
determine implication purely by difference constraints, and use the
standard incremental queueing for bounds propagators/implication by bounds.

We can further improve upon the algorithms above by taking into account
fixedness of variables, and reasoning better on which difference constraints
in $C'$ can be implied by a constraint addition.

Any fixed variable $x$ where $D(x) = \{d\}$ acts similarly to
the dummy variable $v_0$ since any path $x \rightsquigarrow y$
implies a path $v_0 \stackrel{-d}{\rightarrow} x \rightsquigarrow
y$,
and similarly $y \rightsquigarrow x$ implies the
path $y \rightsquigarrow x \stackrel{d}{\to} v_0$.
Once a variable is fixed and we have updated the bounds
caused by the fixing, then the variable plays no further role.
We can ignore fixed variables $x$ in \textsc{IncSat}, \textsc{IncImp},
\textsc{IncLB} and \textsc{IncUB} without compromising
correctness using a result analogous to \Cref{thm:diff}.
Hence edges to and from fixed variables need not be considered in these
algorithms.

After every addition of a constraint we need to check the constraints
in $C'$ for implication.  We can do better than checking every one if we keep
track of what changes have been made by the latest constraints addition.

Cotton and Maler~\cite{Cotton:Maler:06} show how we can restrict the
shortest path calculations in \textsc{IncImp} to those that may
actually decrease a shortest path.
There is a shorter path from $x$ to $y$ via new edge $(u,v,d)$ if and only if
there is a shorter path from $x$ to $v$ or $u$ to $y$ via this edge.
If we calculate $\delta_v^\leftarrow$ (as opposed
to $\delta_u^\leftarrow$) and $\delta_u^\rightarrow$
we can modify Dijkstra's algorithm to restrict attention to those nodes
that give a shorter path using the edge $(u,v,d)$ (see~\cite{Cotton:Maler:06}
for details).
We can calculate the weight of a shortest path from $x$ to $y$ via $(u,v,d)$
as $\delta_v^\leftarrow(x) - d + \delta_u^\rightarrow(y)$.

Note that $wSP(x,y) = w + \pi(y) - \pi(x)$ where $w$ is the weight of a
shortest path
from $x$ to $y$ in
the reduced cost graph. Clearly $w \geq 0$ and hence $wSP(x,y) \geq \pi(y) -
\pi(x)$.
Hence if $\pi(y) - \pi(x) > d$ we know that $x - y \leq d$ is not implied
by the constraints. Note also that $\pi$ only changes by \emph{reducing} the
values at particular nodes (See \textsc{IncSAT}).
Hence if $\pi(y) - \pi(x) > d$ then at some future
time $\pi'(y) - \pi'(x) > d$ unless the potential function value at $y$
changed.

Given both the above observations we can improve the checking of implication
of constraints in $C'$ by only checking constraints $x - y \leq d \in C'$
if $\delta_u^\rightarrow(y)$ made use of the edge $(u,v,d)$
in its calculation, and if $\pi(y) - \pi(x) > d$
in the past, then $\pi(y)$ has changed.

\section{Difference Logic in Huub}
\label{sec:huub}

Huub~\cite{Huub, cp2025c} is a modern CP solver developed in Rust that is based on the IPASIR-UP~\cite{fazekas-sat23} interface to modern SAT solvers. This section provides specifics about the interaction of a difference logic component in a modern CP solver, about pre-processing and simplification, and proposes several control options that are evaluated in the experiment section.

\subsection{Difference Logic Representation}

In Huub, difference logic is represented by a graph $G=(V,E,E')$, where $V\subseteq\mathcal{V}$ is a subset of the variables, and $E$ is the set of active edges $u\glob d v$ representing active difference constraints. Additionally, $E'$ is a set of implied edges $u \impl bd v$ representing the implied (half-reified) difference constraint $b\half u-v\leq d$ for an unknown Boolean variable $b$. These edges will either be dropped if $b\mapsto\bot$, or moved to the active edges $E$ if $b\mapsto\top$. Note that this representation is enough to cover all kinds of difference constraints, including (half-)reified equality and inequality constraints based on the following mapping:
\begin{itemize}
    \item $x-y\leq d$: Directly added to $E$.
    \item $b\half x-y\leq d$: Directly added to $E'$.
    \item $b\Leftrightarrow x-y\leq d$: Transformed to $b \half x-y\leq d$ and $\neg b\half y-x\leq-d-1$.
    \item $x-y=d$: One variable is a transformation of the other (called a view in the solver), no need to represent as a constraint.
    \item $b\half x-y=d$: Transformed to $b\half x-y\leq d$ and $b\half y-x\leq-d$.
    \item $x-y\neq d$: Transformed to $b\half  x-y\leq d-1$ and $\neg b\half y-x\leq-d-1$ for a new Boolean variable $b$.
    \item $b\half x-y\neq d$: Transformed to $b\half c\lor e$, $\neg c\lor\neg e$, $c\half x-y\leq d-1$ and $e\half y-x\leq-d-1$ for new Boolean variables $c$ and $e$.
    \item $b\Leftrightarrow x-y=d$: Transformed to $b\half x-y=d$ and $\neg b\half x-y\neq d$, then further transformation as defined above.
\end{itemize}

Therefore, all versions of potentially (half-)reified difference, equality, and inequality constraints on two variables can be represented in difference logic. However, for equality and inequality, additional Boolean variables need to be introduced, and multiple difference constraints are generated. This might hurt performance, but it might also allow better propagation within the difference logic graph. Therefore, a control parameter for the difference logic component is introduced:
\begin{description}
    \item[\bfseries Level 0:] Difference logic is deactivated.
    \item[\bfseries Level 1:] Include global and (half-)reified difference constraints.
    \item[\bfseries Level 2:] Additionally include half-reified equality constraints.
    \item[\bfseries Level 3:] Additionally include inequality and reified equality constraints.
\end{description}

Depending on the selected level, the included constraints are captured when the model is built. Internally, each node $v$ is associated with a list of incoming active edges $E_v^+$, outgoing active edges $E_v^-$, incoming implied edges ${E'}_v^+$, and outgoing implied edges ${E'}_v^-$. Additionally, each Boolean variable $b$ that is used in implied edges is associated with its implied edges $E_b$. During search, these edges dynamically move from implied to active, or are dropped from the implied edges. However, these changes need to be reverted if the solver backtracks. For this purpose, two types of trailed lists are used:
\begin{description}
    \item[\bfseries Trailed Appending:] Active edges are stored in lists that trail the current active length $\ell$. A new element is stored at index $\ell$ (the array is 0-indexed), and backtracking resets the active length $\ell$. Therefore, recently added elements are not part of the active array any more after backtracking. Note that they are not actually removed, but just overwritten on future extensions. This design allows both setting implied edges to active, and adding difference constraints from external sources.
    \item[\bfseries Trailed Closing:] Implied edges are stored in a type of trailed todo-list. Elements with lower index than the trailed value $c$ are marked as closed, while other elements are still considered open. If an element $i$ with $i\geq c$ is closed, its position is swapped with $c$, and $c$ is increased by 1. On backtracking, $c$ is decreased again, moving the element back into the active section. Note that the order of active elements is not preserved, but closing and backtracking are done in constant time. Additionally, since closed elements are kept at the beginning of the list, trailed appending could be combined with this type for dynamic addition of implied edges. Also consider that each implied edge $u \impl bd v$ is contained in three lists: outgoing from $u$, incoming to $v$, and implied by $b$. Therefore, each edge is associated with its current index in all three lists, updating them on every swap. Then, from each list, closing the edge can be triggered in constant time for all three lists containing the edge.
\end{description}

While in theory a priority queue that allows dynamic priority updates or even a Fibonacci heap is required to achieve the given complexity results, in practice a simple priority queue where duplicate entries become stale and are skipped has shown to be more efficient and is used in the implementation.

\subsection{Initial Simplification}

A modern CP solver typically solves a problem in two major stages. First, a simplification stage tries to reduce the problem as much as possible without going into search, by propagating constraints at the root level, identifying variables that are just projections of other variables, and simplifying constraints when possible. Only when no further simplification is possible, search is started. Since simplification is only done once, it is very important for the performance of a solver to simplify as much as possible. Difference logic allows  simplification that goes beyond what propagators for the individual constraints could do, and may be able to significantly reduce the complexity of the difference logic graph, or even prove infeasibility without starting search. Already before simplification starts for the difference logic component, Huub identifies variables that are views of other variables, i.e., $y$ is a linear transformation of another variable $x$ defined by $x=k\cdot y+c$.

The first time difference logic is called, the whole graph needs to be checked for consistency, and the potential function needs to be defined. While starting with an empty graph and iteratively adding all active edges would be possible, it is more efficient to deal with the whole graph at once, applying the steps in \Cref{alg:init}.

\begin{algorithm2e}[t]
   \caption{\textsc{Initialization}}
   \label{alg:init}
   \SetKw{KwAll}{all}
   \SetKw{KwAnd}{and}
   \KwIn{$G = (V,E,E')$ a difference logic graph.}
   \KwOut{A simplified difference logic graph $G'$ and a valid potential function $\pi$, or \emph{SUBSUMED}, or raise a \emph{CONFLICT}}
   $\pi$ := \textsc{BellmanFord}($G$)\tcp*{\emph{CONFLICT} on negative cycle}
   $G'$ := \textsc{PropagateBounds}($G$, $\pi$)\tcp*{\emph{CONFLICT} on bounds or Booleans}
   $G'$ := \textsc{PropagateBooleans}($G'$, $\pi$)\tcp*{\emph{CONFLICT} on bounds}
   $G'$ := \textsc{RemoveNodes}($G'$)\;
   $G'$ := \textsc{Johnsons}($G'$, $\pi$)\tcp*{\emph{CONFLICT} on Boolean variables}
   $G'$ := \textsc{RemoveNodes}($G'$)\;
   $G'$ := \textsc{RemoveBooleans}($G'$)\;
   \If{$G'$ is empty}{\Return{\emph{SUBSUMED}}}
   \Return{$G'$, $\pi$}
\end{algorithm2e}

The first step is to run the Bellman Ford algorithm from an imaginary initial vertex $v_0$ that is connected via an edge $v_0\glob 0 v$ to each other vertex in $V$. It computes the shortest path from $v_0$ to each vertex $v\in V$ in $\cO(|V|\cdot|E|)$ while being able to deal with negative edge weights. First, it allows us to detect cycles of negative length (which would lead to a shortest path of $-\infty$), leading to a root node proof of infeasibility. Second, setting $\pi(v)$ to the length of the shortest path to $v$ is a valid potential function, which is returned if no failure is detected~\cite{cormen2022introduction}.

Next, an initial propagation of bounds is performed using \textsc{IncLB} and \textsc{IncUB} setting the sets of nodes to evaluate $\mathcal{V}_l$ and $\mathcal{V}_u$ to the set of all nodes in the difference logic graph $V$. This entails removing edges from $E'$ that are implied by bounds, or setting $b\mapsto\bot$ for edges in $E'$ where the inverse is implied. Initial propagation of Boolean variables evaluates all Boolean variables $b$ that imply edges in $E'$ which already have a fixed value (either from the propagation of bounds or externally). Implied edges are dropped if $b\mapsto\bot$, or moved to $E$ by calling \textsc{IncSat} if $b\mapsto\top$.

In a next step, nodes are removed if they already have a fixed value, or if there are no edges connecting them to the remainder of the graph (neither in $E$ nor $E'$). In the first case, implied edges might still exist, e.g., $u$ is fixed, but $u \impl bd v\in E'$ for an unknown Boolean variable $b$. Now the implied edge reduces to an implied bound constraint for $v$, $b\Rightarrow v\geq \theta(u)-d$, which is passed back to the solver to be handled directly (analogously $b\Rightarrow u\leq \theta(v)+d$ if $v$ is fixed).

The last major step in the initialization is to compute all pairwise shortest paths in the active graph by repeated execution of Dijkstra's algorithm, which together with the initialization of the potential function by Bellman Ford constitutes Johnson's algorithm~\cite{cormen2022introduction}. It runs in time complexity $O(|V|^2\log|V|+|V||E|)$, which can be very significant for large graphs, but can lead to substantial simplifications of the difference logic graph which benefit all subsequent evaluations during search. The simplifications performed with the resulting matrix of shortest paths between all vertices $\mathcal{D}_{uv}$ are:
\begin{description}
    \item[\bfseries Removal of redundant global edges:]
        $u\glob d v$ is removed if $d>\mathcal{D}_{uv}$, additionally, at most one edge from $u$ to $v$ is kept if multiple edges $u\glob{\mathcal{D}_{uv}} v$ are defined.
    \item[\bfseries Removal of redundant implied edges:]
        $u\impl bd v$ is removed if $d\geq\mathcal{D}_{uv}$.
    \item[\bfseries Fixing Boolean variables to $\bot$:]
        Addition of $u\impl bd v$ would lead to a negative cycle if $-d>\mathcal{D}_{uv}$, leading to fixing $b\mapsto\bot$.
    \item[\bfseries Node unification:]
        $\mathcal{D}_{uu}=0$ entails that there is a cycle of length $0$. The only valid solution to such a cycle is to set $u-v=d$ for each edge $u\glob dv$ in the cycle. Therefore, variable $v$ can be replaced with a view $v\mapsto u-d$, collapsing all variables in the cycle to a single variable.
\end{description}

Finally, nodes that got fixed or isolated during the previous step are removed again, and Boolean variables not connected to any implied edges any more (because all of them got removed) are also removed from the infrastructure in the graph. Afterwards, the first iterative simplification is triggered.

\subsection{Iterative Simplification}

Since simplification of other propagators can allow further simplification in difference logic, these iterative simplification steps are executed until no propagator can simplify any further. The iterative simplification performs the following steps:
\begin{itemize}
    \item Merge unified variables
    \item Propagation
    \item Removal of nodes and Boolean variables
    \item Check for subsumption
\end{itemize}

Already before initialization, variables that are just linear transformations of other variables are replaced by views. However, further views can be identified during simplification both by the difference logic component and other propagators. Additionally, in difference logic, multiple variables that are views of the same underlying variable can be represented by the same node in the difference logic graph. Assume $v=u+c$, even if $u$ and $v$ are originally not connected by any difference constraints, they can be represented by the same node in the difference logic graph. Simply replace every edge $v\glob dw$ by $u\glob{d-c} w$, analogously for edges in the opposite direction and implied edges. Only views with a multiplicative factor like $v=2\cdot u$ are represented as separate nodes, while all additive transformations are merged. This check is performed first in each iterative simplification.

At the end of propagation in each iterative simplification, again fixed and isolated nodes are removed, and irrelevant Boolean variables are removed. Finally, if all nodes have been removed, subsumption is returned. Otherwise, new iterative simplification might be triggered. If the difference logic component is not subsumed, is finally transformed into its propagator form for the search stage of the solver. In this stage, nodes are no longer removed, as all changes now need to be trailed to be able to revert them. Only the addition of edges to $E$ and the removal of edges from $E'$ are possible from this moment on as described in the trailing infrastructure.

\subsection{Propagation}

During search, propagation of the difference logic component is triggered every time a bound of any node changes, or any of the Boolean variables that imply edges in $E'$ is fixed. These changes are communicated by advisors~\cite{advisors}: The difference logic component registers for updates of all node variables and relevant Boolean variables, and gets notified of each of their changes. Therefore, no exhaustive search of bound changes over the nodes is required, but $\mathcal{V}_l$ and $\mathcal{V}_u$ for \textsc{IncLB} and \textsc{IncUB} as well as the set of fixed Boolean variables $\mathcal{B}$ are directly built based on advisor calls.

The propagation of bound changes first executes \textsc{IncLB} and \textsc{IncUB}, followed by checking all implied edges connected to nodes with bound updates if they are implied by the bounds (remove them from $E'$) or falsified by the bounds (remove them from $E'$ and propagate the corresponding Boolean $b\mapsto\bot$).

The propagation of Boolean changes entails dropping an edge from $E'$ if the Boolean is set to $\bot$, or moving it from $E'$ to $E$ if set to $\top$, which leads to checking for conflicts and potentially updating $\pi$ via \textsc{IncSat}, and optionally checking the impact on further implied edges via \textsc{IncImp}. Note that this check via \textsc{IncImp} can be turned on or off via a parameter. While skipping this step can lead to conflicts that are only discovered later in the search (e.g., \textsc{IncImp} would propagate $b\mapsto\bot$ for $x\impl bd y$, instead $b$ is later set to true and \textsc{IncSat} runs into a conflict), running \textsc{IncImp} is the most costly part of propagation, and the benefit of skipping it can outweigh the cost even though \textsc{IncImp} implements the improvement by Cotton and Maler~\cite{Cotton:Maler:06} calculating $\delta_v^\leftarrow$ and $\delta_u^\rightarrow$ as described in~\Cref{sec:optimizations}.

Additionally, propagation of bound changes and fixed Boolean variables is independent of each other even though they both operate on the same graph and one kind of propagation might trigger the other as a consequence. This entails that the propagations can be triggered separately, and therefore with different priorities \texttt{PrioBounds} and \texttt{PrioBools}.

\section{Experiments}\label{sec:exp}

We consider three implementations in Huub in the experiments:
\textsf{sp}: only separate propagators for each constraint (the standard
approach);
\textsf{gp}: a global propagator using \textsf{IncSat}
and \textsf{IncImp} on difference and reified difference constraints,
and \textsf{IncLB} and \textsf{IncUB} to handle bounds.
\textsf{gp+s}: \textsf{gp} plus the additional simplification stage using Johnson's algorithm followed by the removal of redundant edges and nodes. We additionally evaluate several different parameter settings for \textsf{gp+s}.
We do not consider the base implementation of
\Cref{sec:base} since it 
is an order of magnitude slower than \textsf{gp}
when there are bounds constraints.

All the experiments were carried out on a cluster equipped with Intel Xeon Silver 4314 (2.4 GHz, 24 MB Cache, no hyperthreading) running Ubuntu 22.04 set up for maximum repeatability \cite{cluster_2024}. Each experiment uses a single thread and at most $12.8$ GB of RAM. Huub~\cite{Huub} is evaluated in a development build\footnote{\url{https://github.com/kletzi/huub/tree/feat/difference_logic_benchmark}} and compared to Chuffed 0.13.2~\cite{chuffed:github} and Google OR-Tools CP-SAT 9.15.6755~\cite{cpsatlp} as the two of the best performing solvers in recent MiniZinc challenges~\cite{stuckey2014minizinc}. We specifically compare to other solvers that are applicable to general problems specified in MiniZinc, while other solvers implementing a global difference propagator (see \Cref{sec:related}) are not general purpose finite domain solvers, but more specialized to specific domains.
Full tables of the results on all instances are available in the online supplementary material~\cite{data}.

\subsection{Worst case behavior for separate propagators}
\label{sec:worst}

\Cref{ex:n3} shows a problem where the use of single propagators
could be terrible compared to the global propagator \textsf{gp} 
just for bounds propagation.
In the first experiment, we validate this experimentally,
comparing the separate propagators \textsf{sp} and the
global propagator \textsf{gp} on this problem. \textsf{sp} uses the standard implementation of linear inequality propagators in Huub.
For different number of variables $n$, we applied
consecutive lower bound updates of $y_0$ from $(k-1)n$ to $kn$, $k \in 1..10$,
and measured the time needed to reach these 10 fixpoints.
The results are show in \Cref{fig:worst}. The \textsf{gp} times are all below $70$ ms, while the \textsf{sp} times show the expected cubic growth.
One can see from the results that very bad propagation
behavior can occur in practical solvers when using single propagators, while a
global propagator only needs to run once to reach a fixpoint.

\begin{figure}
\centering
\begin{tikzpicture}
  \begin{axis}[
    xlabel={$n$},
    ylabel={Time (s)},
    legend pos=north west,
    grid=major,
    grid style={dashed, gray!30},
    xtick distance=100,
    width=12cm,
    height=8cm,
    mark size=2pt,
  ]

    \addplot[red, mark=square, thick] coordinates {
      (100,  0.39200)
      (200,  3.56167)
      (300,  11.27867)
      (400,  25.21700)
      (500,  46.87067)
      (600,  80.67833)
      (700,  124.95367)
      (800,  183.87200)
      (900,  258.61667)
      (1000, 351.17133)
    };
    \addlegendentry{sp}
 
    \addplot[blue, mark=o, thick] coordinates {
      (100,  0.00100)
      (200,  0.00400)
      (300,  0.00800)
      (400,  0.01400)
      (500,  0.02000)
      (600,  0.02700)
      (700,  0.03500)
      (800,  0.04500)
      (900,  0.05600)
      (1000, 0.06800)
    };
    \addlegendentry{gp}
 
  \end{axis}
\end{tikzpicture}
\caption{Propagation times (s) to reach fixpoint for \Cref{ex:n3} 
as a function of $n$}
\label{fig:worst}
\end{figure}
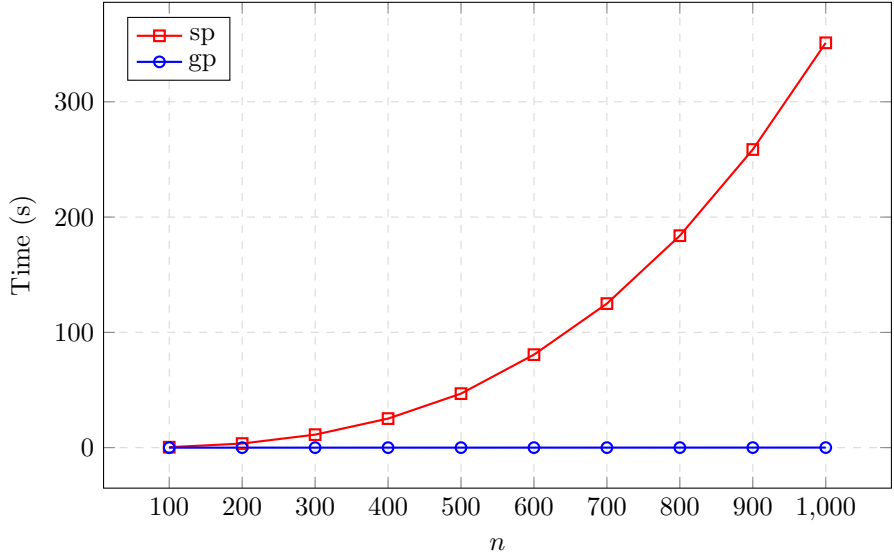

\Cref{fig:worst_log} shows the comparison in logarithmic scale, including not just the propagation time, but also the initialization and simplification time. For \textsf{gp}, initialization and simplification actually takes the majority of the time, but still stays below $0.25$ s for all experiments. For \textsf{sp}, initialization time is similar to \textsf{gp} in absolute values, but barely notable compared to the huge propagation time.

Additionally, we compare \textsf{gp+s}, which includes the full simplification using Johnson's algorithm. While this algorithm also has cubic runtime like \textsf{sp}, runtime is still orders of magnitude lower compared to \textsf{sp}. First, the cubic cost in this case occurs only once during simplification, while for \textsf{sp} it repeatedly occurs in propagation. Second, for \textsf{gp+s} the cubic runtime occurs within one algorithm in the propagator, while \textsf{sp} shows a cubic invokation of different propagators which is orders of magnitude more costly. Therefore, even for worst-case examples, simplification shows acceptable runtime cost. Note that a similar behavior to \textsf{gp+s} occurs in a modified version of the worst-case example where all non-zero arcs are negative due to the initial computation of the potential function using Bellman-Ford.

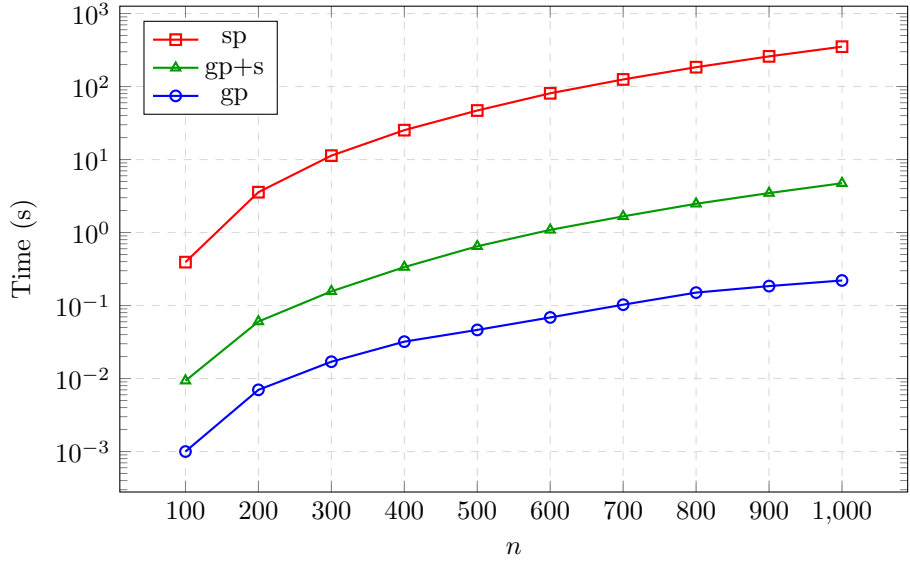
\begin{figure}
\centering
\begin{tikzpicture}
  \begin{axis}[
    xlabel={$n$},
    ylabel={Time (s)},
    legend pos=north west,
    grid=major,
    grid style={dashed, gray!30},
    xtick distance=100,
    width=12cm,
    height=8cm,
    mark size=2pt,
    ymode=log,
  ]
  
    \addplot[red, mark=square, thick] coordinates {
      (100,  0.39300)
      (200,  3.56967)
      (300,  11.29867)
      (400,  25.25200)
      (500,  46.92367)
      (600,  80.75933)
      (700,  125.06033)
      (800,  184.01233)
      (900,  258.79233)
      (1000, 351.38800)
    };
    \addlegendentry{sp}
 
    \addplot[green!60!black, mark=triangle, thick] coordinates {
      (100,  0.00933)
      (200,  0.06033)
      (300,  0.15633)
      (400,  0.33600)
      (500,  0.64800)
      (600,  1.08733)
      (700,  1.66533)
      (800,  2.47700)
      (900,  3.47067)
      (1000, 4.73400)
    };
    \addlegendentry{gp+s}
 
    \addplot[blue, mark=o, thick] coordinates {
      (100,  0.00100)
      (200,  0.00700)
      (300,  0.01700)
      (400,  0.03200)
      (500,  0.04633)
      (600,  0.06867)
      (700,  0.10267)
      (800,  0.15033)
      (900,  0.18467)
      (1000, 0.22033)
    };
    \addlegendentry{gp}
 
  \end{axis}
\end{tikzpicture}
\caption{Total times (s) in log scale to reach fixpoint for \Cref{ex:n3} including initialization time}
\label{fig:worst_log}
\end{figure}

\subsection{Benchmark Problems}

In the following sections, we evaluate three sets of benchmark instances to evaluate the effect of several configuration choices for difference logic, and to compare the performance with and without difference logic.

To test difference logic on a wide range of different problems, we evaluate the repository of MiniZinc challenge benchmarks~\cite{challenge} that contains all previous problems used in the yearly MiniZinc challenge since 2008. The problems are filtered to remove those that do not contain any difference constraints as captured by the propagator in this paper for any instance. Problems occurring in different years are only kept in their most recent appearance. The resulting benchmark set contains 626 instances and is run with the same timeout as available in the MiniZinc challenge, which is 20 minutes.

Further we use two specific problems that contain difference constraints as a major component of their model for evaluation. The second benchmark is a set of complex resource-constrained project scheduling benchmarks using producer / consumer constraints \texttt{ProdCons}~\cite{schutt2016producer} and precedences. The model is taken directly from this reference using the \texttt{dAfter} decomposition. The instances range from 10 to 1000 activities using 5 resources. The small instance set for analysing different configurations using 10 instances per size contains 120 instances in total, while the full data set contains 1080 instances.

The third benchmark in the evaluation is \texttt{RCPSPmax}~\cite{schutt2013solving}, which is a variant of resource-constraint project scheduling with minimal and maximal time lags. Already the decision variant of the problem is NP-hard. Again, there is a small data set using 10 instances per size for 120 instances in total, and a full data set with 1130 instances from the CD, UBO, and SM datasets in literature ranging from 10 to 1000 activities using 5 resources (per dataset and size at most 100 instances).

\subsection{Evaluation of Different Configurations}

In this section we evaluate a range of difference logic configurations on the whole MiniZinc challenge benchmark and the small instance sets for \texttt{ProdCons} and \texttt{RCPSPmax}. The following choices are evaluated in all possible 162 combinations:
\begin{itemize}
    \item \textbf{Level of captured constraints:} 1 (global and (half-)reified difference constraints), 2 (+ half-reified equalities), 3 (+ inequalities and reified equalities)
    \item \textbf{Priorities:} Lowest (0), medium (2), highest (4) individually for bound propagation and Boolean propagation
    \item \textbf{Implied check:} Use of \textsc{IncImp} for early detection of conflicts (on or off)
    \item \textbf{Bound propagation level:} Lazy + lifted (\texttt{Lazy}), Eager + lifted (\texttt{Lifted}), Eager not lifted (\texttt{Simple})
\end{itemize}

\subsubsection{MiniZinc Challenge Benchmarks}

\begin{figure}[t]
    \centering
    \begin{tikzpicture}
\begin{axis}[
    xlabel={Points},
    ytick={1, 2, 3},
    yticklabels={{Level 1}, {Level 2}, {Level 3}},
    xmajorgrids=true,
    grid style=dashed,
    boxplot/draw direction=x,
    height=\dimexpr 1cm * 3 + 1cm\relax,
    width=0.88\linewidth,
    enlarge y limits=0.1,
    y dir=reverse,
    legend style={draw=none},
    legend image post style={draw=none},
]
\addplot[
    boxplot prepared={
        lower whisker=300.37,
        lower quartile=306.08,
        median=309.47,
        upper quartile=313.798,
        upper whisker=321.91
    },
    fill=blue!60, draw=blue!80!black, fill opacity=0.7
] coordinates {};
\addplot[
    boxplot prepared={
        lower whisker=298.53,
        lower quartile=303.835,
        median=307.515,
        upper quartile=312.415,
        upper whisker=321.28
    },
    fill=red!60, draw=red!80!black, fill opacity=0.7
] coordinates {};
\addplot[
    boxplot prepared={
        lower whisker=290.32,
        lower quartile=300.215,
        median=303.86,
        upper quartile=308.495,
        upper whisker=317.87
    },
    fill=green!80!black, draw=green!60!black, fill opacity=0.7
] coordinates {};
\draw[red, thick, dashed] (axis cs:310, \pgfkeysvalueof{/pgfplots/ymin}) -- (axis cs:310, \pgfkeysvalueof{/pgfplots/ymax});
\end{axis}
\end{tikzpicture}
    \caption{Points on the MiniZinc benchmark for different levels of constraint acquisition}
    \label{fig:level}
\end{figure}
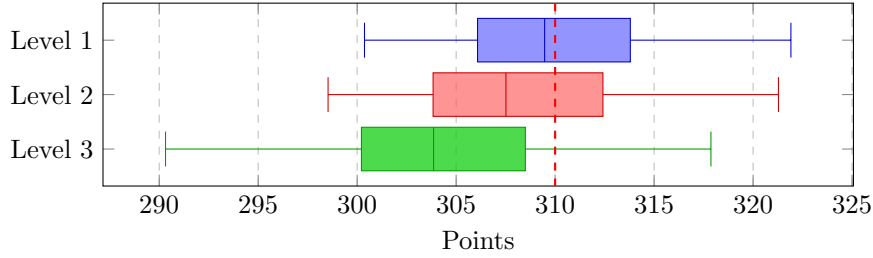

Since there are 162 configurations to evaluate, we use the tool \texttt{mzn-bench}\footnote{\url{https://github.com/MiniZinc/mzn-bench}} to calculate a score of each configuration in comparison to the baseline configuration (level 0, difference logic is off, i.e. the standard solver \textsf{sp}) using the scoring system of the MiniZinc challenge~\cite{challenge}. For each instance, this score awards a point to the solver providing a better solution status (optimal / unsatisfiable if the other solver did not reach a proof, better objective if no solver reached a proof), or splits the points proportionally to the runtime if both solvers provide a proof of optimality or unsatisfiability).

For each of the following boxplots, we collect all results sharing a particular configuration setting (with varying other configurations) and show minimum, maximum, quartiles (Q1, Q3), and median in each graph. A dotted line at 310 points indicates the amount of points needed to match the performance of the baseline configuration.

\Cref{fig:level} shows the impact of the level of constraints captured by difference logic. Clearly, level 3 is outperformed by the other levels, with even Q3 below the baseline. Level 1 outperforms the other levels, indicating that the extra effort to include equality and inequality constraints by introducing extra variables generates an overhead that does not pay off by reaching better propagation in the difference logic graph.

\begin{figure}[t]
    \centering
   \begin{tikzpicture}
\begin{axis}[
    xlabel={Points},
    ytick={1, 2, 3},
    yticklabels={{Lowest}, {Medium}, {Highest}},
    xmajorgrids=true,
    grid style=dashed,
    boxplot/draw direction=x,
    height=\dimexpr 1cm * 3 + 1cm\relax,
    width=0.88\linewidth,
    enlarge y limits=0.1,
    y dir=reverse,
    legend style={draw=none},
    legend image post style={draw=none},
]
\addplot[
    boxplot prepared={
        lower whisker=297.61,
        lower quartile=307.168,
        median=312.045,
        upper quartile=316.315,
        upper whisker=321.91
    },
    fill=blue!60, draw=blue!80!black, fill opacity=0.7
] coordinates {};
\addplot[
    boxplot prepared={
        lower whisker=290.32,
        lower quartile=302.567,
        median=305.74,
        upper quartile=309.007,
        upper whisker=317.66
    },
    fill=red!60, draw=red!80!black, fill opacity=0.7
] coordinates {};
\addplot[
    boxplot prepared={
        lower whisker=295.92,
        lower quartile=302.033,
        median=305.265,
        upper quartile=310.255,
        upper whisker=319.73
    },
    fill=green!80!black, draw=green!60!black, fill opacity=0.7
] coordinates {};
\draw[red, thick, dashed] (axis cs:310, \pgfkeysvalueof{/pgfplots/ymin}) -- (axis cs:310, \pgfkeysvalueof{/pgfplots/ymax});
\end{axis}
\end{tikzpicture}
    \caption{Points on the MiniZinc benchmark for different priorities for bound propagation}
    \label{fig:priorities_bounds}
\end{figure}
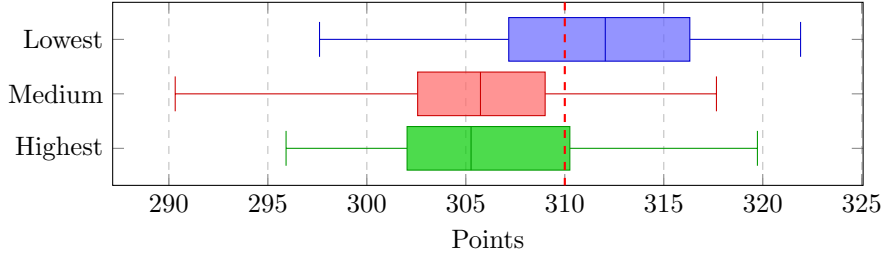

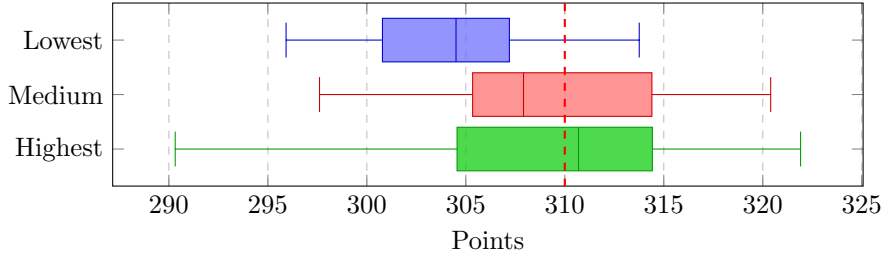
\begin{figure}[t]
    \centering
   \begin{tikzpicture}
\begin{axis}[
    xlabel={Points},
    ytick={1, 2, 3},
    yticklabels={{Lowest}, {Medium}, {Highest}},
    xmajorgrids=true,
    grid style=dashed,
    boxplot/draw direction=x,
    height=\dimexpr 1cm * 3 + 1cm\relax,
    width=0.88\linewidth,
    enlarge y limits=0.1,
    y dir=reverse,
    legend style={draw=none},
    legend image post style={draw=none},
]
\addplot[
    boxplot prepared={
        lower whisker=295.92,
        lower quartile=300.788,
        median=304.505,
        upper quartile=307.2,
        upper whisker=313.76
    },
    fill=blue!60, draw=blue!80!black, fill opacity=0.7
] coordinates {};
\addplot[
    boxplot prepared={
        lower whisker=297.61,
        lower quartile=305.345,
        median=307.915,
        upper quartile=314.4,
        upper whisker=320.4
    },
    fill=red!60, draw=red!80!black, fill opacity=0.7
] coordinates {};
\addplot[
    boxplot prepared={
        lower whisker=290.32,
        lower quartile=304.558,
        median=310.695,
        upper quartile=314.423,
        upper whisker=321.91
    },
    fill=green!80!black, draw=green!60!black, fill opacity=0.7
] coordinates {};
\draw[red, thick, dashed] (axis cs:310, \pgfkeysvalueof{/pgfplots/ymin}) -- (axis cs:310, \pgfkeysvalueof{/pgfplots/ymax});
\end{axis}
\end{tikzpicture}
    \caption{Points on the MiniZinc benchmark for different priorities for Boolean propagation}
    \label{fig:priorities_bool}
\end{figure}

\Cref{fig:priorities_bounds} shows the results for bound propagation priorities, while \Cref{fig:priorities_bool} shows Boolean propagation priorities. The results are very interesting as they show that for bound propagation the lowest setting is beneficial, outperforming the baseline in most cases, while medium and highest mostly reduce the number of points. In contrast, for Boolean propagation the lowest setting is clearly inferior, and the highest setting, even considering some negative outliers, can provide the best results. These results show that separating the priorities of the different propagation stages is valuable for the performance of difference logic.

\begin{figure}[t]
    \centering
   \begin{tikzpicture}
\begin{axis}[
    xlabel={Points},
    ytick={1, 2},
    yticklabels={{\textsc{IncImp} off}, {\textsc{IncImp} on}},
    xmajorgrids=true,
    grid style=dashed,
    boxplot/draw direction=x,
    height=\dimexpr 1cm * 2 + 1.5cm\relax,
    width=0.88\linewidth,
    enlarge y limits=0.2,
    y dir=reverse,
    legend style={draw=none},
    legend image post style={draw=none},
]
\addplot[
    boxplot prepared={
        lower whisker=295.92,
        lower quartile=305.45,
        median=310.43,
        upper quartile=314.57,
        upper whisker=321.91
    },
    fill=blue!60, draw=blue!80!black, fill opacity=0.7
] coordinates {};
\addplot[
    boxplot prepared={
        lower whisker=290.32,
        lower quartile=300.88,
        median=305.13,
        upper quartile=309.18,
        upper whisker=317.1
    },
    fill=red!60, draw=red!80!black, fill opacity=0.7
] coordinates {};
\draw[red, thick, dashed] (axis cs:310, \pgfkeysvalueof{/pgfplots/ymin}) -- (axis cs:310, \pgfkeysvalueof{/pgfplots/ymax});
\end{axis}
\end{tikzpicture}
    \caption{Points on the MiniZinc benchmark depending on the use of \textsc{IncImp}}
    \label{fig:incimp}
\end{figure}
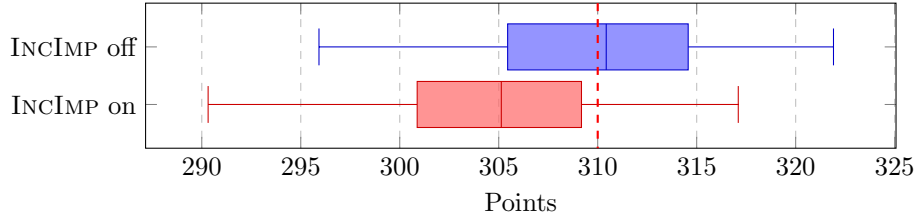

\Cref{fig:incimp} shows the difference between activation and deactivation of \textsc{IncImp}. It is clear that activation mostly reduces the scores, while deactivation shows much better results. This indicates that the overhead of the \textsc{IncImp} check is in general not worth the effort even though earlier propagation of Boolean variables can be achieved.

\begin{figure}[t]
    \centering
   \begin{tikzpicture}
\begin{axis}[
    xlabel={Points},
    ytick={1, 2, 3},
    yticklabels={{\texttt{Simple}}, {\texttt{Lifted}}, {\texttt{Lazy}}},
    xmajorgrids=true,
    grid style=dashed,
    boxplot/draw direction=x,
    height=\dimexpr 1cm * 3 + 1cm\relax,
    width=0.88\linewidth,
    enlarge y limits=0.1,
    y dir=reverse,
    legend style={draw=none},
    legend image post style={draw=none},
]
\addplot[
    boxplot prepared={
        lower whisker=290.32,
        lower quartile=302.42,
        median=306.345,
        upper quartile=309.007,
        upper whisker=321.91
    },
    fill=blue!60, draw=blue!80!black, fill opacity=0.7
] coordinates {};
\addplot[
    boxplot prepared={
        lower whisker=295.92,
        lower quartile=302.777,
        median=308.255,
        upper quartile=314.94,
        upper whisker=321.28
    },
    fill=red!60, draw=red!80!black, fill opacity=0.7
] coordinates {};
\addplot[
    boxplot prepared={
        lower whisker=298.42,
        lower quartile=304.045,
        median=308.885,
        upper quartile=313.735,
        upper whisker=321.05
    },
    fill=green!80!black, draw=green!60!black, fill opacity=0.7
] coordinates {};
\draw[red, thick, dashed] (axis cs:310, \pgfkeysvalueof{/pgfplots/ymin}) -- (axis cs:310, \pgfkeysvalueof{/pgfplots/ymax});
\end{axis}
\end{tikzpicture}
    \caption{Points on the MiniZinc benchmark depending on explanation type}
    \label{fig:explanation}
\end{figure}
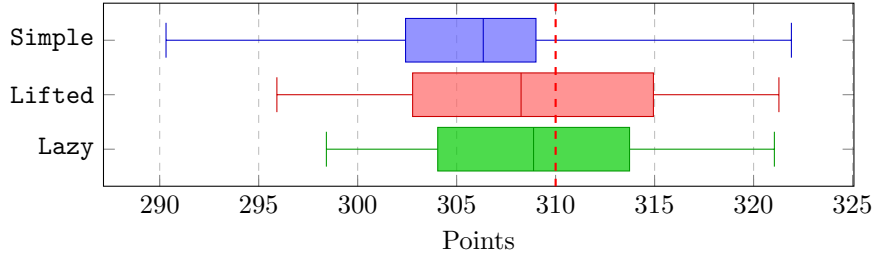

Finally, \Cref{fig:explanation} shows the comparison of the different types of explanations for Boolean variables propagated by bound changes. While \texttt{Lifted} and \texttt{Lazy} show similar performance, it is clear that \texttt{Simple}, not relying on any lifting, is inferior in most configurations. However, it has a very wide span including both the highest and lowest overall scores.

Next, to validate the individual findings and check for any strong interactions between different settings, we look at the top configurations overall. The top 5 configurations, all above 320 points, all agree on \textsc{IncImp} off and lowest priority for bound propagation. Four of them use highest Boolean propagation priority, one medium. Two configurations use level 1, three use level 2. The explanation types are mixed with 2 \texttt{Lazy}, 2 \texttt{Lifted}, and 1 \texttt{Simple}. These match well with the individual evaluations above.

\subsubsection{Scheduling Benchmarks}

The configuration evaluation was conducted on a very diverse set of benchmarks, ensuring robustness of the results. However, since the importance of difference logic can vary for different problem domains, we also evaluate all configurations on the small data sets for \texttt{ProdCons} and \texttt{RCPSPmax} to see the effects in more detail on problem domains where difference logic plays a larger role.

Note that for \texttt{ProdCons} all configurations outperform the baseline, while for \texttt{RCPSPmax} 157 out of 162 configurations outperform the baseline, showing a clear advantage of using difference logic on these problems. This section summarizes findings on the benchmarks, for the corresponding graphs seen \Cref{sec:prodcons} and \Cref{sec:rcpspmax}.

For both problems there are no additional constraints to capture at levels 2 or 3, therefore the results only show noise.

More differences occur for priorities, especially for \texttt{ProdCons}. For bound propagation, the lowest setting actually has the highest median of points, but also the lowest Q1, Q3, and maximum, showing a less clear overall picture. For Boolean propagation, the picture from the general benchmark is actually reversed, with the lowest setting outperforming the medium and highest setting. Priority settings seem to have less impact for \texttt{RCPSPmax}, with a light trend of lower scores for lowest priorities in both categories. The best overall results are achieved with highest bound and medium Boolean priority. These diverse results show that the best priority setting is more problem-dependent.

\textsc{IncImp} off is again clearly ahead compared to the activated setting. While the result is already strong for \texttt{ProdCons}, for \texttt{RCPSPmax} it is actually so strong that all results using \textsc{IncImp} score strictly below all results deactivating it, making this finding the most strongly supported finding among the configuration experiments.

Finally, the choice for explanations is again more problem-specific. \texttt{ProdCons} shows highest points for \texttt{Simple}, closely followed by \texttt{Lifted}, both clearly ahead of \texttt{Lazy}, while \texttt{RCPSPmax} shows \texttt{Lazy} ahead of \texttt{Simple} and \texttt{Lifted} last.

Overall, we choose the configuration with level 1, lowest bound priority, highest Boolean priority, no use of \textsc{IncImp}, and \texttt{Lifted} explanations as our final configuration, since this configuration excels in the diverse MiniZinc challenge benchmark ($320.95$ points, maximum achieved $321.95$), combines all preferred configuration options from the individual investigations on this benchmark, but also performs well on the specific problems ($67.17$ points with maximum achieved $67.87$ on \texttt{ProdCons}, $62.12$ points with maximum achieved $65.44$ on \texttt{RCPSPmax}). For \texttt{RCPSPmax} the top scores require higher priorities for bound propagation and \texttt{Lazy} explanations, but these options did not show consistent enough performance to be selected overall.

\subsection{Comparison of Top Configurations on Full Instance Set}

This section shows the comparison of Huub with and without difference logic to other state-of-the-art solvers on the MiniZinc benchmark and the full benchmarks on \texttt{ProdCons} and \texttt{RCPSPmax}. Specifically, \texttt{Huub} (\textsf{sp}) denotes running without difference logic, \texttt{Huub} (\textsf{gp+s}) denotes running the best configuration of difference logic according to the previous section including simplification, and \texttt{Huub} (\textsf{gp}) denotes the same configuration, but with initial simplification deactivated. For \texttt{ProdCons} and \texttt{RCPSPmax}, all versions of Huub are run with the flag \texttt{--reason-eager 1} to avoid triggering a rare bug in the interface to CaDiCaL on unsatisfiability proofs.

\begin{table}[t]
    \centering
    \begin{tabular}{l|rr|rr|rrr|rr|r}
    \toprule
        & \multicolumn{2}{c|}{Optimal} & \multicolumn{2}{c|}{Unsat.} & \multicolumn{3}{c|}{Satisfied} & Unkn. & Error & \multicolumn{1}{c}{Points} \\
& nr & time & nr & time & \#to & nr & time & nr & nr & vs. Huub \\
\midrule
Huub (\textsf{sp}) & $353$ & $105.62$ & $\mathbf{21}$ & $\mathbf{23.32}$ & $184$ & $50$ & $29.37$ & $16$ & $2$ & $310.00$ \\
Huub (\textsf{gp+s}) & $352$ & $101.74$ & $20$ & $74.53$ & $189$ & $50$ & $33.73$ & $\mathbf{13}$ & $2$ & $320.95$ \\
Huub (\textsf{gp}) & $353$ & $101.95$ & $20$ & $74.81$ & $187$ & $50$ & $37.27$ & $14$ & $2$ & $312.94$ \\
Chuffed & $333$ & $115.45$ & $19$ & $24.01$ & $183$ & $49$ & $62.80$ & $25$ & $17$ & $251.42$ \\
CP-SAT & $\mathbf{370}$ & $\mathbf{98.00}$ & $19$ & $120.64$ & $169$ & $\mathbf{52}$ & $32.22$ & $15$ & $1$ & $\mathbf{343.83}$ \\
\bottomrule
    \end{tabular}
    \caption{Comparison of different solvers on the MiniZinc challenge benchmark}
    \label{tab:compare_challenge}
\end{table}

\Cref{tab:compare_challenge} shows the comparison on the MiniZinc challenge benchmark. The different columns show the number of optimally solved instances, their average runtime, the number of instances proven unsatisfiable, their average runtime, the number of instances where a satisfying solution was found, but timeout occurred (\#to), the number of satisfaction instances solved, their average runtime, the number of unknown instances (no solution found within runtime, Unkn.), the number of instances aborted due to errors,
and the points according to the MiniZinc challenge in direct comparison to \texttt{Huub} (\textsf{sp}).

The evaluation shows that CP-SAT, the winner of the MiniZinc challenge, is still the strongest solver on this general benchmark. As it is a portfolio solver combining different strategies, it is well suited for this benchmark as different strategies are expected to excel at different groups of instances. While it has the highest number of optimal solutions and solved satisfaction instances, as well as the highest number of points, it is outperformed by all three versions of Huub on the number of unsatisfiable instances, and Huub with difference logic has the lowest number of unknown instances on the benchmark. Further, difference logic reduces the gap between Huub and CP-SAT in points by around one third, showing the valuable contribution of this component on the general benchmark.

Without initial simplification, Huub with difference logic is still better than the original in points, but only just, showing the importance of these initial steps in the overall performance. Overall, difference logic does not increase the number of optimal or proven unsatisfiable instances on this benchmark, but the benefit seems to come from finding better solutions faster on more complex instances. Chuffed is inferior to the other competitors, showing fewer optimal solutions, more unknown instances, and more errors.

\begin{table}[t]
    \centering
    \begin{tabular}{l|rr|rr|r|rr|r}
    \toprule
        & \multicolumn{2}{c|}{Optimal} & \multicolumn{2}{c|}{Unsatisfiable} & \multicolumn{1}{c|}{Satisfied} & Unknown & Error & \multicolumn{1}{c}{Points} \\
& nr & time & nr & time & \#to & nr & nr & vs. Huub \\
\midrule
Huub (\textsf{sp}) & $605$ & $59.00$ & $392$ & $9.99$ & $43$ & $39$ & $1$ & $540.00$ \\
Huub (\textsf{gp+s}) & $\mathbf{660}$ & $\mathbf{44.76}$ & $\mathbf{399}$ & $\mathbf{3.86}$ & $19$ & $\mathbf{2}$ & $0$ & $\mathbf{637.50}$ \\
Huub (\textsf{gp}) & $631$ & $64.92$ & $395$ & $7.93$ & $6$ & $48$ & $0$ & $573.40$ \\
Chuffed & $519$ & $55.06$ & $391$ & $29.26$ & $62$ & $108$ & $0$ & $442.50$ \\
CP-SAT & $559$ & $87.93$ & $391$ & $26.25$ & $7$ & $123$ & $0$ & $364.52$ \\
\bottomrule
    \end{tabular}
    \caption{Comparison of different solvers on full \texttt{ProdCons} benchmark}
    \label{tab:compare_prodcons}
\end{table}

\Cref{tab:compare_prodcons} shows the comparison on the full \texttt{ProdCons} benchmark, where Huub already outperforms the competitors without difference logic. Using difference logic, the performance is greatly boosted. More instances are solved to optimality in shorter time, more instances are proven unsatisfiable in shorter time, and only 2 unknown instances remain. These results show the great benefits of difference logic on such scheduling benchmarks. The comparison to the version without simplification also shows the importance of this component. Without simplification, there is still a large improvement in the number of optimal solutions, however, there is also an increase in the number of unknown solutions. The majority of the improvements of using difference logic require working on the simplified graph.

\begin{table}[t]
    \centering
    \begin{tabular}{l|rr|rr|r|rr|r}
    \toprule
        & \multicolumn{2}{c|}{Optimal} & \multicolumn{2}{c|}{Unsatisfiable} & \multicolumn{1}{c|}{Satisfied} & Unknown & Error & \multicolumn{1}{c}{Points} \\
& nr & time & nr & time & \#to & nr & nr & vs. Huub \\
\midrule
Huub (\textsf{sp}) & $876$ & $16.59$ & $\mathbf{172}$ & $1.98$ & $49$ & $32$ & $1$ & $550.00$ \\
Huub (\textsf{gp+s}) & $\mathbf{883}$ & $12.53$ & $\mathbf{172}$ & $\mathbf{<0.01}$ & $44$ & $\mathbf{31}$ & $0$ & $\mathbf{586.55}$ \\
Huub (\textsf{gp}) & $880$ & $11.66$ & $\mathbf{172}$ & $0.26$ & $37$ & $41$ & $0$ & $570.78$ \\
Chuffed & $858$ & $23.84$ & $170$ & $1.62$ & $45$ & $57$ & $0$ & $490.51$ \\
CP-SAT & $853$ & $20.77$ & $\mathbf{172}$ & $6.07$ & $55$ & $49$ & $1$ & $436.93$ \\
\bottomrule
    \end{tabular}
    \caption{Comparison of different solvers on full \texttt{RCPSPmax} benchmark}
    \label{tab:compare_rcpspmax}
\end{table}

Finally, we compare the different options on \texttt{RCPSPmax} in \Cref{tab:compare_rcpspmax}. Again, Huub can outperform the competitors already without difference logic, but further improve with difference logic. Again, difference logic provides more optimal and proven unsatisfiable solutions in less runtime and fewer unknown instances. The benefit of simplification is mostly seen in the number of unknown instances. Remarkably, difference logic with simplification is so efficient in detecting unsatisfiability on this benchmark that the average time is below $0.01$ second, so most of these instances are identified during simplification.

Overall, the results show that difference logic makes Huub an even stronger competitor on a very general set of instances, and helps the solver to excel on scheduling benchmarks where it shows excellent performance in comparison to other state-of-the-art solvers. The simplification stage proves to be an important part of the practical success of difference logic.


\section{Related Work}
\label{sec:related}

Early work on difference logic constraints (also called temporal constraints) was done by
Dechter et al.~\cite{dechter1991temporal}, which introduced the \emph{simple temporal constraints network} for difference logic constraints and the formal framework for reasoning about them using shortest path algorithms.
This network is the same directed constraint graph~$G_C$ as defined in \Cref{def:constraintgraph} on page~\pageref{def:constraintgraph}, in which they also encode bounds constraints on variables with a dummy node.
This work formed a basis for later work for finite domain solvers, e.g., IBM ILOG CP Optimizer.

As mentioned in the introduction, SMT solvers~\cite{smt, smt2} like, e.g., Z3 and cvc5, have been treating difference logic constraints on a global level and reasoning about them using specialized shortest path algorithms~\cite{Cotton:Maler:06} whether the current system is satisfiable or the entailment or disentailment of reified constraints.
Our work adapts these SMT methods for finite domain solvers with and without clause learning.

In the context of CP, the previous published version of this work~\cite{feydy2008global}
in 2008 introduced the global difference logic propagator to the CP community
and implemented it in the G12 finite domain solver engine~\cite{G12}.
Before that difference logic constraints were modelled as separated
linear constraints.

Schutt and Stuckey~\cite{schutt2016producer} implemented the global difference
logic propagator~\cite{feydy2008global} in the Lazy Clause Generation~\cite{feydy2009lazy} solver Chuffed~\cite{chu2011diss, chuffed:github} for solving problems with producer/consumer constraints, but do not provide details how
any propagation is explained, except they refer to the difference logic theory propagator~\cite{Cotton:Maler:06} and mentioned that bounds literals $\lit{\pm x\leq d}$ are treated separately as described in~\cite{feydy2008global}.
They only considered reified constraints $b \leftrightarrow x - y\leq d$, which needed to be statically added to the model before solving, whereas our solver allows dynamic additions of half-reified and reified constraints during the search.

More recently, Hebrard~\cite{hebrard2025disjunctive} introduces the new CP/SAT hybrid solver
Tempo currently specialized for disjunctive scheduling, in which they implement a global difference logic propagator.
Tempo lazily creates bound literals in a branch only when necessary and they only exist
in the subtree rooted at this branch, whereas in our solver they currently exist from the start, but the Huub architecture allows for dynamic trailed additions, which we can explore in the future.
They also separate lower and upper bound propagation and difference logic propagation as we do.
However, their bound propagators uses a version of Dijkstra with a simple FIFO queue,
which directly works on negative edges and has the worst case runtime complexity of $\cO(nm)$ opposed to ours $\cO(n\log n + m)$.

Their difference logic propagator supports dynamic additions of half-reified difference logic constraints, in which two constraints can also be expressed in a XOR-relationship by using the same Boolean variable, but one in the negative form.
They use this relationship to model the non-overlap of two tasks sharing the same disjunctive resource.
For example, the relationship between two tasks \texttt{A} and \texttt{B} models that either \texttt{A} must finish before the start of \texttt{B} or \texttt{B} must finish before the start of \texttt{A}.
Our propagator can support the same relationship, but we also consider equality and inequality constraints of two variables.
Implication of such half-reified constraints are only checked regarding the variable's bounds and their explanation is the same as our explanation \texttt{Lifted}.
They note that checking implication via the difference logic graph is too costly,
but do not provide any empirical results as we do.

They empirically evaluated Tempo on various disjunctive scheduling problems with the objective of minimizing the makespan,
in which they use the global difference logic propagator and the edge-finding rule from the global disjunctive propagator (also called unary or no-overlap propagator) to identify new difference logic constraints (``edges'') between tasks, which are then shared with the difference logic propagator.
In our case, we do not have any propagator that infers new difference logic constraints and shares them with the global difference logic propagator at the moment.
However, our propagator is implemented in a more general solver and evaluated on a wider set of problems involving different objective functions.

\ignore{
\as{My notes about Tempo comparison.}
\begin{itemize}
    \item depending on constraint, it uses lcg explanation or backward explanations
    \item \as{DONE} only considers disjunctive constraints
    \item \as{DONE} only considers makespan objective
    \item bound literals are stored in a sorted list
    \item bounds literals are not linked to any Boolean variable
    \item \as{DONE} considers half-reified constraints
    \item \as{DONE} can use the same Boolean to model XOR relations between any two edges, it is useful for non-overlapping constraints (task a must be before or after task b)
    \item uses a reversible directed graph
    \item \as{DONE} uses the same separation of bound literals in the global difference logic propagator as in \cite{feydy2008global}.
    \item \as{DONE} uses a version of Dijkstra algorithm to update the bounds, but with a simple FIFO queue and a vertex can enter multiple times the queue having the same complexity as the Bellman-Ford algorithm, but can directly work on graphs with negative cycles
    \item \as{DONE} do not infer all entailed edges, because it is too costly. It seems that it is only infers entailed edges if they are entail by variable bounds.
    \item \as{DONE} uses edge-finder algorithm for disjunctive constraint to find new edges
    \item they mentioned that they ``don't use the strongest possible consistency in the different logic reasoner''
    \item they are adding transitive edges to the reversible graph structure
\end{itemize}
}

Aries~\cite{bit2023enhancing} is a dedicated clause learning solver for solving only disjunctive scheduling problems minimizing the makespan that uses a global difference propagator.
The non-overlapping constraints between two tasks sharing the same disjunctive resource are modelled by two reified difference constraints using the same binary variable.
The makespan is modelled with difference logic constraints with any other task, which is the same in our case if the model declares inequalities between makespan and the end of each task.
It uses an adaption of the Bellman-Ford algorithm described in~\cite{cesta1996gaining} to propagate the bounds, which has a worst-case runtime complexity of $\cO(mn)$ compared to $\cO(n\log n + m)$.
Their bounds' propagation is equivalent to our \texttt{Simple} explanation.
\ignore{
\as{My notes about Aries comparison.}
\begin{itemize}
    \item Specialized Clause Learning solver for disjunctive scheduling
    \item considers reified constraints, but not half-reified
    \item Exploits XOR-relation for modelling non-overlap of tasks
    \item Uses an adaption from \cite{cesta1996gaining} for propagating bounds, which is based on the Bellman-Ford algorithm
\end{itemize}
}

IBM ILOG CP Optimizer~\cite{laborie2018ibm} is a specialized solver for planning and scheduling that uses a global difference propagator for propagating the time bounds of the start and end of (optional) interval variables representing the start and end time of an activity.
Since the solver is a closed-source software only limited information about the implementation and internal workings are known.
In~\cite{laborie2018ibm}, they mentioned that their difference logic graph is based on the graph in~\cite{dechter1991temporal} extended by the presence statuses of interval variables.
Bound constraints are explicitly handled in~\cite{dechter1991temporal} compared to our separation.
The improved Bellman-Ford algorithm~\cite{cherkassky1996shortest} is used for the initial propagation of the graph, whereas an extended version of the algorithm described in~\cite{cesta1996gaining} is used for subsequent propagation.
Both have a worst-case runtime complexity $\cO(nm)$, whereas our algorithm for the subsequent propagation has only $\cO(n\log n + m)$.
Compared to our implementation, they also consider optional interval variables leading to different logic constraints, in that the variables can be optional.
They perform time-bound adjustment on those optional variables if their presence status is implied by the other, even in the case if there presence status is unknown.
For that reason, they maintain a global ``implication'' graph about the presence statuses of all optional interval variables.
Furthermore, they identify negative cycles of such difference constraints, for which the presence status of the involved optional variables are implied for each other, and set the presence statuses to ``absent'' in such a case.
Our propagator would be able to handle difference logic constraints with optional variables $x-y\leq d$, but can only propagate them when they are present.
In our case, such a constraint would be modelled by a half-reified constraint $c \Rightarrow x-y \leq d$ and the constraint $c = a\wedge b$ where the Boolean variables $a$ and $b$ represent the presence status of the optional variables $x$ and $y$.

\ignore{
\as{My notes about CPO comparison.}
\begin{itemize}
    \item Using a simple temporal network of Dechter\cite{dechter1991temporal} extended by presence statuses to reason about precedence constraints between start/end relationship of interval variables where the ``distance'' can be a variable or an expression.
    \item a node in the temporal network represent a start or end time-points of interval variables and the length of an edge is the minimal delay
    \item closed-source software
    \item initial propagation of the temporal network is performed by an improved Bellman-Ford algorithm from Cherkassky et al. and an incremental propagation is performed when a time-bound changed, new edge is added, and a new implication detected using an extension of the algorithm by Cesta et al.
    \item they have a logical network about the presence statuses of the interval constraints, which is used for deciding whether bounds of a temporal constraints between conditional interval variables can be performed
    \item time-bounds are not only propagated for ``present'' temporal constraints, but also for ``unknown'' constraints if the present status of the involved variables is implied in at least one direction
    \item propagates the absent value between optional interval variables if a positive length cycle amongst them exist.
\end{itemize}
}

Google OR-Tools CP-SAT solver~\cite{cpsatlp} is an LCG solver that can also make use of a simplex algorithm for solving LP relaxations of linear constraints in the model.
While difference logic constraints are internally represented as two-variables per inquality (TVPI) constraints (i.e. $ax + by \leq d$ where $a$ and $b$ are constant numbers) and no global difference logic propagator exists, it applies a transitive closure algorithm on unit TVPIs (i.e. TVPI's where $a$ and $b$ is $1$ or $-1$) to tighten the bound on the constraint (i.e. $d$), but only at the root level.
For that purpose, the algorithm builds a directly acyclic graph by adding two edges $x\to y$ and $y\to x$ for each constraint and then computes the transitive closure using a specific variable order.
However, it comes with two caveats:
First, if there is a cycle amongst the set of constraints then the algorithm is not executed.
Second, the algorithm has a work limit for computing the transitive closure, once the limit is reached the algorithm stops.
We note that the cycle detection does not consider the length of it and, therefore, does not identify infeasible system.
Beside this algorithm, difference logic constraints and their (half-)reified versions are treated as separated propagators, but propagated in an efficient order to minimize the workload using their specific variable ordering.
Regarding the internal LP solver, CP-SAT does not automatically add difference logic constraints to the LP relaxation, but decides it depending on the linear constraints in the model.
We also ran CP-SAT with settings that always add the difference logic constraints to the LP relaxation, but the results did not improve.

To sum up, the works~\cite{hebrard2025disjunctive,bit2023enhancing,laborie2018ibm,schutt2016producer} implement a version of the global difference logic propagator in a finite domain constraint solver with or without clause learning,
which reasons about the consistency of the system, implication of (half-)reified constraints, and the variables' bounds.
All of them are specific implementation of a group of problems and only tested on them: 
disjunctive scheduling~\cite{bit2023enhancing,hebrard2025disjunctive},
scheduling problems involving producer/consumer constraints~\cite{schutt2016producer}, and
planning and scheduling problems~\cite{laborie2018ibm},
whereas our propagator is implemented in a general purpose solver and tested on a wide range of problems that include difference logic constraints.
In the case of \cite{schutt2016producer,laborie2018ibm}, implementation details are not known.

\section{Conclusion}
\label{sec:conc}

Difference constraints appear widely in constraint programming models,
and the default representation as individual propagators has some
known bad behavior. In this paper we explore 
how a global treatment of difference constraints can improve propagation. 
We explore a number of different design choices. Overall we find that 
applying the global propagator at lowest priority for bounds, highest priority for booleans, not running the \textsc{IncImp} algorithm, and applying initial simplification leads to the best results overall.  

We note that while SMT~\cite{smt,smt2} solvers treat difference 
constraints globally, the naive importing of their methodology is
impractical for finite domain propagation solving because of the
relative importance of bounds propagation for FD. For problems with
bounds constraints the approach outlined in Section 4.1 is an order
of magnitude slower than \textsf{gp}.


The global difference solver opens the possibility of improving the behaviour of CP solvers by instrumenting globals to take advantage of known difference constraints and/or
learn new difference logic consequences of propagation, as in the disjunctive propagator for Tempo. This remains as important future work. 

\bibliography{quellen}

\newpage
\appendix

\section{Configuration Experiments on \texttt{ProdCons}}
\label{sec:prodcons}

This section shows the graphs for the configuration experiments on the small \texttt{ProdCons} instance set. All configurations outperform the baseline which requires more than 60 points.

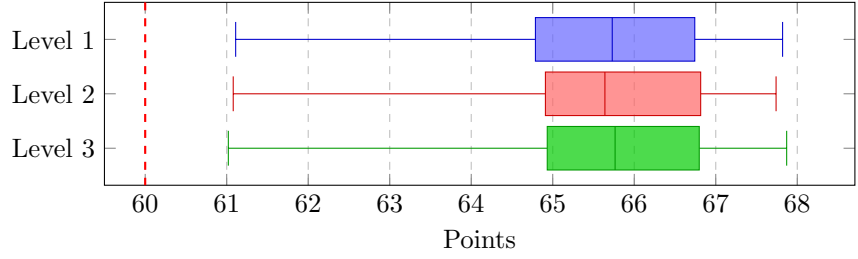
\begin{figure}[H]
    \centering
    \begin{tikzpicture}
\begin{axis}[
    xmin=59.5,
    xlabel={Points},
    ytick={1, 2, 3},
    yticklabels={{Level 1}, {Level 2}, {Level 3}},
    xmajorgrids=true,
    grid style=dashed,
    boxplot/draw direction=x,
    height=\dimexpr 1cm * 3 + 1cm\relax,
    width=0.88\linewidth,
    enlarge y limits=0.1,
    y dir=reverse,
    legend style={draw=none},
    legend image post style={draw=none},
]
\addplot[
    boxplot prepared={
        lower whisker=61.11,
        lower quartile=64.7875,
        median=65.73,
        upper quartile=66.7425,
        upper whisker=67.82
    },
    fill=blue!60, draw=blue!80!black, fill opacity=0.7
] coordinates {};
\addplot[
    boxplot prepared={
        lower whisker=61.08,
        lower quartile=64.91,
        median=65.64,
        upper quartile=66.815,
        upper whisker=67.74
    },
    fill=red!60, draw=red!80!black, fill opacity=0.7
] coordinates {};
\addplot[
    boxplot prepared={
        lower whisker=61.02,
        lower quartile=64.9325,
        median=65.765,
        upper quartile=66.7975,
        upper whisker=67.87
    },
    fill=green!80!black, draw=green!60!black, fill opacity=0.7
] coordinates {};
\draw[red, thick, dashed] (axis cs:60, \pgfkeysvalueof{/pgfplots/ymin}) -- (axis cs:60, \pgfkeysvalueof{/pgfplots/ymax});
\end{axis}
\end{tikzpicture}
    \caption{Points on the small \texttt{ProdCons} benchmark for different levels of constraint acquisition}
\end{figure}

\begin{figure}[H]
    \centering
   \begin{tikzpicture}
\begin{axis}[
    xmin=59.5,
    xlabel={Points},
    ytick={1, 2, 3},
    yticklabels={{Lowest}, {Medium}, {Highest}},
    xmajorgrids=true,
    grid style=dashed,
    boxplot/draw direction=x,
    height=\dimexpr 1cm * 3 + 1cm\relax,
    width=0.88\linewidth,
    enlarge y limits=0.1,
    y dir=reverse,
    legend style={draw=none},
    legend image post style={draw=none},
]
\addplot[
    boxplot prepared={
        lower whisker=61.99,
        lower quartile=64.2025,
        median=66.13,
        upper quartile=66.74,
        upper whisker=67.29
    },
    fill=blue!60, draw=blue!80!black, fill opacity=0.7
] coordinates {};
\addplot[
    boxplot prepared={
        lower whisker=61.02,
        lower quartile=64.9075,
        median=65.665,
        upper quartile=67.1025,
        upper whisker=67.73
    },
    fill=red!60, draw=red!80!black, fill opacity=0.7
] coordinates {};
\addplot[
    boxplot prepared={
        lower whisker=62.64,
        lower quartile=65.1875,
        median=65.615,
        upper quartile=67.0525,
        upper whisker=67.87
    },
    fill=green!80!black, draw=green!60!black, fill opacity=0.7
] coordinates {};
\draw[red, thick, dashed] (axis cs:60, \pgfkeysvalueof{/pgfplots/ymin}) -- (axis cs:60, \pgfkeysvalueof{/pgfplots/ymax});
\end{axis}
\end{tikzpicture}
    \caption{Points on the small \texttt{ProdCons} benchmark for different priorities for bound propagation}
\end{figure}
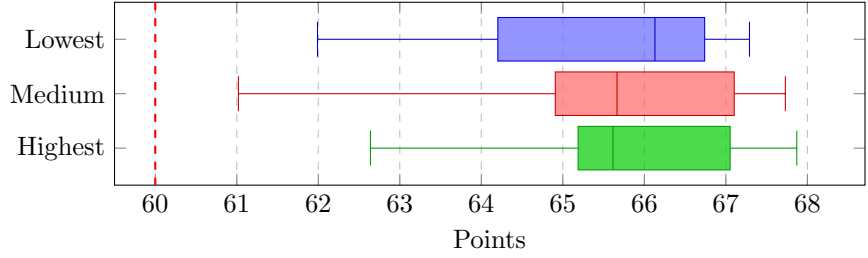

\begin{figure}[H]
    \centering
   \begin{tikzpicture}
\begin{axis}[
    xmin=59.5,
    xlabel={Points},
    ytick={1, 2, 3},
    yticklabels={{Lowest}, {Medium}, {Highest}},
    xmajorgrids=true,
    grid style=dashed,
    boxplot/draw direction=x,
    height=\dimexpr 1cm * 3 + 1cm\relax,
    width=0.88\linewidth,
    enlarge y limits=0.1,
    y dir=reverse,
    legend style={draw=none},
    legend image post style={draw=none},
]
\addplot[
    boxplot prepared={
        lower whisker=64.97,
        lower quartile=65.57,
        median=66.66,
        upper quartile=67.4675,
        upper whisker=67.87
    },
    fill=blue!60, draw=blue!80!black, fill opacity=0.7
] coordinates {};
\addplot[
    boxplot prepared={
        lower whisker=61.99,
        lower quartile=64.7475,
        median=65.475,
        upper quartile=66.105,
        upper whisker=67.82
    },
    fill=red!60, draw=red!80!black, fill opacity=0.7
] coordinates {};
\addplot[
    boxplot prepared={
        lower whisker=61.02,
        lower quartile=64.1875,
        median=65.08,
        upper quartile=66.3175,
        upper whisker=67.64
    },
    fill=green!80!black, draw=green!60!black, fill opacity=0.7
] coordinates {};
\draw[red, thick, dashed] (axis cs:60, \pgfkeysvalueof{/pgfplots/ymin}) -- (axis cs:60, \pgfkeysvalueof{/pgfplots/ymax});
\end{axis}
\end{tikzpicture}
    \caption{Points on the small \texttt{ProdCons} benchmark for different priorities for Boolean propagation}
\end{figure}
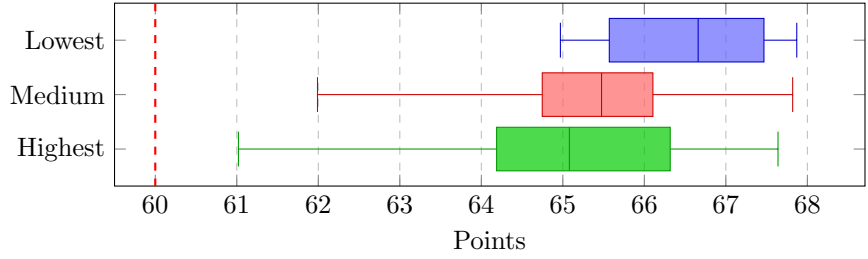

\begin{figure}[H]
    \centering
   \begin{tikzpicture}
\begin{axis}[
    xmin=59.5,
    xlabel={Points},
    ytick={1, 2},
    yticklabels={{\textsc{IncImp} off}, {\textsc{IncImp} on}},
    xmajorgrids=true,
    grid style=dashed,
    boxplot/draw direction=x,
    height=\dimexpr 1cm * 2 + 1.5cm\relax,
    width=0.88\linewidth,
    enlarge y limits=0.1,
    y dir=reverse,
    legend style={draw=none},
    legend image post style={draw=none},
]
\addplot[
    boxplot prepared={
        lower whisker=64.15,
        lower quartile=65.57,
        median=66.18,
        upper quartile=67.06,
        upper whisker=67.87
    },
    fill=blue!60, draw=blue!80!black, fill opacity=0.7
] coordinates {};
\addplot[
    boxplot prepared={
        lower whisker=61.02,
        lower quartile=64.21,
        median=65.18,
        upper quartile=66.1,
        upper whisker=67.66
    },
    fill=red!60, draw=red!80!black, fill opacity=0.7
] coordinates {};
\draw[red, thick, dashed] (axis cs:60, \pgfkeysvalueof{/pgfplots/ymin}) -- (axis cs:60, \pgfkeysvalueof{/pgfplots/ymax});
\end{axis}
\end{tikzpicture}
    \caption{Points on the small \texttt{ProdCons} benchmark depending on the use of \textsc{IncImp}}
\end{figure}
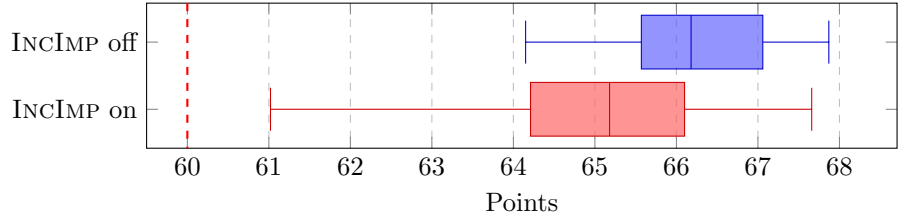

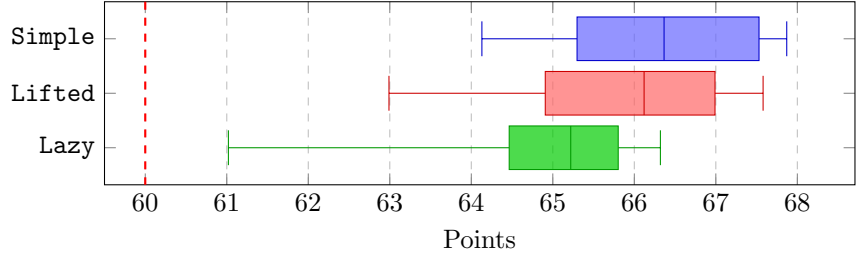
\begin{figure}[H]
    \centering
   \begin{tikzpicture}
\begin{axis}[
    xmin=59.5,
    xlabel={Points},
    ytick={1, 2, 3},
    yticklabels={{\texttt{Simple}}, {\texttt{Lifted}}, {\texttt{Lazy}}},
    xmajorgrids=true,
    grid style=dashed,
    boxplot/draw direction=x,
    height=\dimexpr 1cm * 3 + 1cm\relax,
    width=0.88\linewidth,
    enlarge y limits=0.1,
    y dir=reverse,
    legend style={draw=none},
    legend image post style={draw=none},
]
\addplot[
    boxplot prepared={
        lower whisker=64.13,
        lower quartile=65.2975,
        median=66.365,
        upper quartile=67.53,
        upper whisker=67.87
    },
    fill=blue!60, draw=blue!80!black, fill opacity=0.7
] coordinates {};
\addplot[
    boxplot prepared={
        lower whisker=62.99,
        lower quartile=64.9075,
        median=66.12,
        upper quartile=66.99,
        upper whisker=67.58
    },
    fill=red!60, draw=red!80!black, fill opacity=0.7
] coordinates {};
\addplot[
    boxplot prepared={
        lower whisker=61.02,
        lower quartile=64.465,
        median=65.22,
        upper quartile=65.8025,
        upper whisker=66.32
    },
    fill=green!80!black, draw=green!60!black, fill opacity=0.7
] coordinates {};
\draw[red, thick, dashed] (axis cs:60, \pgfkeysvalueof{/pgfplots/ymin}) -- (axis cs:60, \pgfkeysvalueof{/pgfplots/ymax});
\end{axis}
\end{tikzpicture}
    \caption{Points on the small \texttt{ProdCons} benchmark depending on explanation type}
\end{figure}

\section{Configuration Experiments on \texttt{RCPSPmax}}
\label{sec:rcpspmax}

This section shows the graphs for the configuration experiments on the small \texttt{RCPSPmax} instance set. Almost all configurations outperform the baseline which requires more than 55 points.

\begin{figure}[H]
    \centering
    \begin{tikzpicture}
\begin{axis}[
    xlabel={Points},
    ytick={1, 2, 3},
    yticklabels={{Level 1}, {Level 2}, {Level 3}},
    xmajorgrids=true,
    grid style=dashed,
    boxplot/draw direction=x,
    height=\dimexpr 1cm * 3 + 1cm\relax,
    width=0.88\linewidth,
    enlarge y limits=0.1,
    y dir=reverse,
    legend style={draw=none},
    legend image post style={draw=none},
]
\addplot[
    boxplot prepared={
        lower whisker=52.98,
        lower quartile=57.6325,
        median=59.945,
        upper quartile=62.06,
        upper whisker=65.34
    },
    fill=blue!60, draw=blue!80!black, fill opacity=0.7
] coordinates {};
\addplot[
    boxplot prepared={
        lower whisker=52.92,
        lower quartile=57.8725,
        median=59.91,
        upper quartile=62.2325,
        upper whisker=65.44
    },
    fill=red!60, draw=red!80!black, fill opacity=0.7
] coordinates {};
\addplot[
    boxplot prepared={
        lower whisker=53.12,
        lower quartile=57.655,
        median=60.165,
        upper quartile=62.4125,
        upper whisker=65.39
    },
    fill=green!80!black, draw=green!60!black, fill opacity=0.7
] coordinates {};
\draw[red, thick, dashed] (axis cs:55, \pgfkeysvalueof{/pgfplots/ymin}) -- (axis cs:55, \pgfkeysvalueof{/pgfplots/ymax});
\end{axis}
\end{tikzpicture}
    \caption{Points on the small \texttt{RCPSPmax} benchmark for different levels of constraint acquisition}
\end{figure}
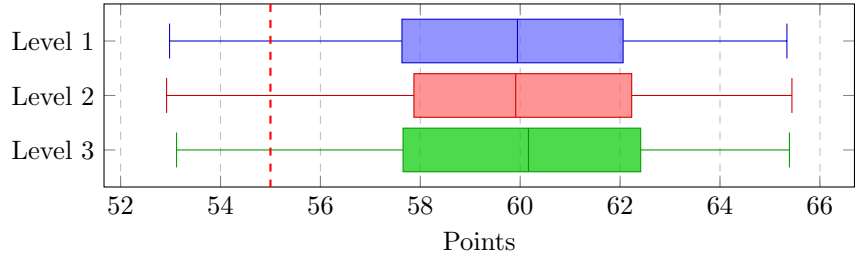

\begin{figure}[H]
    \centering
   \begin{tikzpicture}
\begin{axis}[
    xlabel={Points},
    ytick={1, 2, 3},
    yticklabels={{Lowest}, {Medium}, {Highest}},
    xmajorgrids=true,
    grid style=dashed,
    boxplot/draw direction=x,
    height=\dimexpr 1cm * 3 + 1cm\relax,
    width=0.88\linewidth,
    enlarge y limits=0.1,
    y dir=reverse,
    legend style={draw=none},
    legend image post style={draw=none},
]
\addplot[
    boxplot prepared={
        lower whisker=55.02,
        lower quartile=57.305,
        median=59.345,
        upper quartile=61.9875,
        upper whisker=64.25
    },
    fill=blue!60, draw=blue!80!black, fill opacity=0.7
] coordinates {};
\addplot[
    boxplot prepared={
        lower whisker=52.92,
        lower quartile=56.9475,
        median=60.035,
        upper quartile=62.8825,
        upper whisker=64.57
    },
    fill=red!60, draw=red!80!black, fill opacity=0.7
] coordinates {};
\addplot[
    boxplot prepared={
        lower whisker=54.9,
        lower quartile=58.5275,
        median=60.265,
        upper quartile=62.3975,
        upper whisker=65.44
    },
    fill=green!80!black, draw=green!60!black, fill opacity=0.7
] coordinates {};
\draw[red, thick, dashed] (axis cs:55, \pgfkeysvalueof{/pgfplots/ymin}) -- (axis cs:55, \pgfkeysvalueof{/pgfplots/ymax});
\end{axis}
\end{tikzpicture}
    \caption{Points on the small \texttt{RCPSPmax} benchmark for different priorities for bound propagation}
\end{figure}
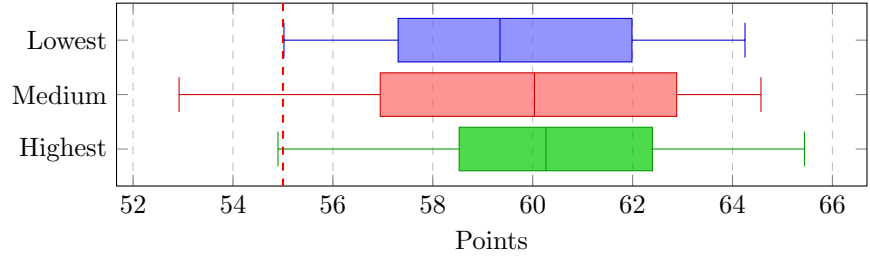

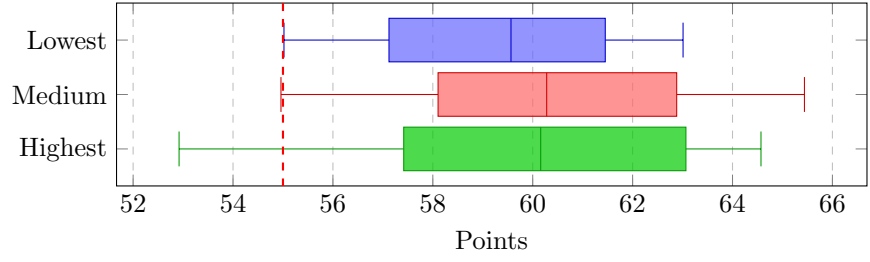
\begin{figure}[H]
    \centering
   \begin{tikzpicture}
\begin{axis}[
    xlabel={Points},
    ytick={1, 2, 3},
    yticklabels={{Lowest}, {Medium}, {Highest}},
    xmajorgrids=true,
    grid style=dashed,
    boxplot/draw direction=x,
    height=\dimexpr 1cm * 3 + 1cm\relax,
    width=0.88\linewidth,
    enlarge y limits=0.1,
    y dir=reverse,
    legend style={draw=none},
    legend image post style={draw=none},
]
\addplot[
    boxplot prepared={
        lower whisker=55.02,
        lower quartile=57.1225,
        median=59.565,
        upper quartile=61.4525,
        upper whisker=63.01
    },
    fill=blue!60, draw=blue!80!black, fill opacity=0.7
] coordinates {};
\addplot[
    boxplot prepared={
        lower whisker=54.96,
        lower quartile=58.105,
        median=60.28,
        upper quartile=62.8825,
        upper whisker=65.44
    },
    fill=red!60, draw=red!80!black, fill opacity=0.7
] coordinates {};
\addplot[
    boxplot prepared={
        lower whisker=52.92,
        lower quartile=57.4175,
        median=60.16,
        upper quartile=63.0675,
        upper whisker=64.57
    },
    fill=green!80!black, draw=green!60!black, fill opacity=0.7
] coordinates {};
\draw[red, thick, dashed] (axis cs:55, \pgfkeysvalueof{/pgfplots/ymin}) -- (axis cs:55, \pgfkeysvalueof{/pgfplots/ymax});
\end{axis}
\end{tikzpicture}
    \caption{Points on the small \texttt{RCPSPmax} benchmark for different priorities for Boolean propagation}
\end{figure}

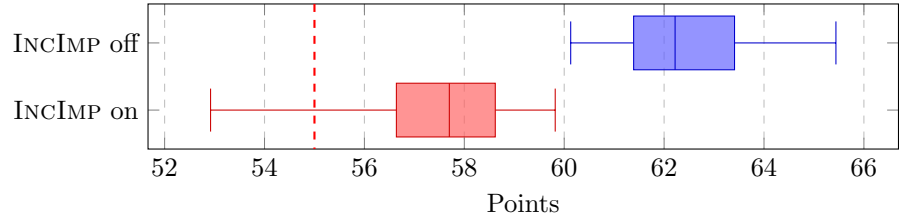
\begin{figure}[H]
    \centering
   \begin{tikzpicture}
\begin{axis}[
    xlabel={Points},
    ytick={1, 2},
    yticklabels={{\textsc{IncImp} off}, {\textsc{IncImp} on}},
    xmajorgrids=true,
    grid style=dashed,
    boxplot/draw direction=x,
    height=\dimexpr 1cm * 2 + 1.5cm\relax,
    width=0.88\linewidth,
    enlarge y limits=0.1,
    y dir=reverse,
    legend style={draw=none},
    legend image post style={draw=none},
]
\addplot[
    boxplot prepared={
        lower whisker=60.13,
        lower quartile=61.39,
        median=62.22,
        upper quartile=63.41,
        upper whisker=65.44
    },
    fill=blue!60, draw=blue!80!black, fill opacity=0.7
] coordinates {};
\addplot[
    boxplot prepared={
        lower whisker=52.92,
        lower quartile=56.64,
        median=57.7,
        upper quartile=58.62,
        upper whisker=59.82
    },
    fill=red!60, draw=red!80!black, fill opacity=0.7
] coordinates {};
\draw[red, thick, dashed] (axis cs:55, \pgfkeysvalueof{/pgfplots/ymin}) -- (axis cs:55, \pgfkeysvalueof{/pgfplots/ymax});
\end{axis}
\end{tikzpicture}
    \caption{Points on the small \texttt{RCPSPmax} benchmark depending on the use of \textsc{IncImp}}
\end{figure}

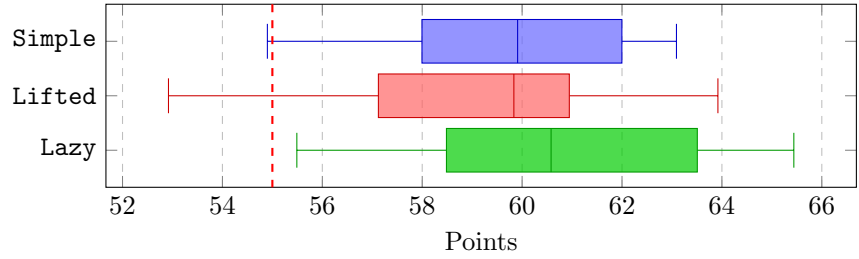
\begin{figure}[H]
    \centering
   \begin{tikzpicture}
\begin{axis}[
    xlabel={Points},
    ytick={1, 2, 3},
    yticklabels={{\texttt{Simple}}, {\texttt{Lifted}}, {\texttt{Lazy}}},
    xmajorgrids=true,
    grid style=dashed,
    boxplot/draw direction=x,
    height=\dimexpr 1cm * 3 + 1cm\relax,
    width=0.88\linewidth,
    enlarge y limits=0.1,
    y dir=reverse,
    legend style={draw=none},
    legend image post style={draw=none},
]
\addplot[
    boxplot prepared={
        lower whisker=54.9,
        lower quartile=57.995,
        median=59.91,
        upper quartile=61.9975,
        upper whisker=63.09
    },
    fill=blue!60, draw=blue!80!black, fill opacity=0.7
] coordinates {};
\addplot[
    boxplot prepared={
        lower whisker=52.92,
        lower quartile=57.1225,
        median=59.835,
        upper quartile=60.945,
        upper whisker=63.92
    },
    fill=red!60, draw=red!80!black, fill opacity=0.7
] coordinates {};
\addplot[
    boxplot prepared={
        lower whisker=55.49,
        lower quartile=58.485,
        median=60.58,
        upper quartile=63.5075,
        upper whisker=65.44
    },
    fill=green!80!black, draw=green!60!black, fill opacity=0.7
] coordinates {};
\draw[red, thick, dashed] (axis cs:55, \pgfkeysvalueof{/pgfplots/ymin}) -- (axis cs:55, \pgfkeysvalueof{/pgfplots/ymax});
\end{axis}
\end{tikzpicture}
    \caption{Points on the small \texttt{ProdCons} benchmark depending on explanation type}
\end{figure}

\section{Submission Information}
\subsection*{Author Contributions}

This paper is an extension of an original conference paper~\cite{feydy2008global} written by Thibaut Feydy, Andreas Schutt and Peter J. Stuckey; Thibaut Feydy has meanwhile left academia.
The extended paper details the full new implementation in a new learning solver Huub~\cite{cp2025c}.
Lucas Kletzander developed the difference logic implementation in Huub  with help from Jip J. Dekker, the lead developer of Huub. Lucas performed experiments and wrote Section 5. The remaining authors were all involved in discussions on algorithmic design and edited the paper substantially.  

\subsection*{Compliance with Ethical Standards}

As a commentary paper there are no conflicts of interest, or other complex ethical standards to apply. Generative AI was only used to generate scripts for data collection and generation of figures and tables. All scripts and outputs were checked by the authors for correctness.

\subsection*{Competing Interests}

Not applicable, although Peter J. Stuckey sits on the Editorial board for Constraints.

\subsection*{Data Availability Declaration}

The solvers used in the evaluation are available as open source. The detailed result tables are available as online supplementary material.


\end{document}